\documentclass[preprint,authoryear,12pt]{elsarticle}
\usepackage{mathrsfs}
\usepackage{graphics,graphicx,epsfig}
\usepackage{amsmath,mathrsfs,amssymb,amsfonts,amsthm}
\usepackage{algorithm,algorithmic}
\usepackage{multirow,bm}
\usepackage[american]{babel}

\newtheorem{theorem}{Theorem}
\newtheorem{lemma}{Lemma}

\DeclareFixedFont{\myfont}{OT1}{ptm}{m}{n}{7pt}
\DeclareFixedFont{\myfontb}{OT1}{ptm}{bx}{n}{7pt}

\def \zerovec {\boldsymbol{0}}

\def \evec {\boldsymbol{e}}

\def \gvec {\boldsymbol{g}}

\def \qvec {\boldsymbol{q}}
\def \uvec {\boldsymbol{u}}

\def \wvec {\boldsymbol{w}}
\def \xvec {\boldsymbol{x}}
\def \yvec {\boldsymbol{y}}
\def \zvec {\boldsymbol{z}}

\def \Amat {\boldsymbol{A}}

\def \Gmat {\boldsymbol{G}}
\def \Hmat {\boldsymbol{H}}
\def \Imat {\boldsymbol{I}}

\def \Qmat {\boldsymbol{Q}}
\def \Xmat {\boldsymbol{X}}
\def \Ymat {\boldsymbol{Y}}

\def \alphavec {\boldsymbol{\alpha}}
\def \betavec {\boldsymbol{\beta}}
\def \xivec {\boldsymbol{\xi}}
\def \epsilonvec {\boldsymbol{\epsilon}}
\def \zetavec {\boldsymbol{\zeta}}

\def \thetavec {\boldsymbol{\theta}}

\def \Dcal {\mathcal{D}}
\def \Lcal {\mathcal{L}}

\def \Xcal {\mathcal{X}}
\def \Ycal {\mathcal{Y}}

\def \Ebb {\mathbb{E}}
\def \Rbb {\mathbb{R}}

\def \wvecbar {\bar{\boldsymbol{w}}}
\def \muvecbar {\bar{\boldsymbol{\mu}}}

\def \argmin {\mathop{\mbox{argmin}}}
\def \st {\mbox{s.t.}}

\journal{Artificial Intelligence Journal}

\begin{document}

\begin{frontmatter}
\title{Optimal Margin Distribution Machine}
\author{Teng Zhang}
\author{Zhi-Hua Zhou\corref{cor1}}
\address{National Key Laboratory for Novel Software Technology,\\
Collaborative Innovation Center of Novel Software Technology and Industrialization\\
Nanjing University, Nanjing 210023, China}
\cortext[cor1]{Email: zhouzh@lamda.nju.edu.cn}

\begin{abstract}
Support vector machine (SVM) has been one of the most popular learning algorithms, with the central idea of maximizing the \textit{minimum margin}, i.e., the smallest distance from the instances to the classification boundary. Recent theoretical results, however, disclosed that maximizing the minimum margin does not necessarily lead to better generalization performances, and instead, the margin distribution has been proven to be more crucial. Based on this idea, we propose a new method, named Optimal margin Distribution Machine (ODM), which tries to achieve a better generalization performance by optimizing the margin distribution. We characterize the margin distribution by the first- and second-order statistics, i.e., the margin mean and variance. The proposed method is a general learning approach which can be used in any place where SVM can be applied, and their superiority is verified both theoretically and empirically in this paper.
\end{abstract}

\end{frontmatter}

\section{Introduction}

Support Vector Machine (SVM)~\citep{Cortes1995Support,Vapnik1995The} has always been one of the most successful learning algorithms. The basic idea is to identify a classification boundary having a large margin for all the training examples, and the resultant optimization can be accomplished by a quadratic programming (QP) problem. Although SVMs have a long history of literatures, there are still great efforts~\citep{Lacoste2013Block,Learning2013Cotter,Takac2013Mini,Jose2013Local,Convex2013Do} on improving SVMs.

It is well known that SVMs can be viewed as a learning approach trying to maximize the \textit{minimum margin} of training examples, i.e., the smallest distance from the examples to the classification boundary, and the margin theory~\citep{Vapnik1995The} provided a good support to the generalization performance of SVMs. It is noteworthy that the margin theory not only plays an important role for SVMs, but also has been extended to interpret the good generalization of many other learning approaches, such as AdaBoost~\citep{Freund1995A}, a major representative of ensemble methods~\citep{Zhou2012Ensemble}. Specifically, Schapire et al.~\citep{Schapire1998Boosting} first suggested margin theory to explain the phenomenon that AdaBoost seems resistant to overfitting; soon after, Breiman~\citep{Breiman1999Prediction} indicated that the minimum margin is crucial and developed a boosting-style algorithm, named Arc-gv, which is able to maximize the minimum margin but with a poor generalization performance. Later, Reyzin et al.~\citep{Reyzin2006How} found that although Arc-gv tends to produce larger minimum margin, it suffers from a poor margin distribution; they conjectured that the margin distribution, rather than the minimum margin, is more crucial to the generalization performance. Such a conjecture has been theoretically studied~\citep{Wang2011A,Gao2012On}, and it was recently proven by Gao and Zhou~\citep{Gao2012On}. Moreover, it was disclosed that rather than simply considering a single-point margin, both the margin mean and variance are important~\citep{Gao2012On, Zhou2014Large}. All these theoretical studies, however, focused on boosting-style algorithms, whereas the influence of the margin distribution for SVMs in practice has not been well exploited.

In this paper, we propose a new method, named Optimal margin Distribution Machine (ODM), which tries to achieve strong generalization performance by optimizing the margin distribution. Inspired by the recent theoretical result~\citep{Gao2012On}, we characterize the margin distribution by the first- and second-order statistics, and try to maximize the margin mean and minimize the margin variance simultaneously. For optimization, we propose a dual coordinate descent method for kernel ODM, and a stochastic gradient descent with variance reduction for large scale linear kernel ODM. Comprehensive experiments on thirty two regular scale data sets and ten large scale data sets show the superiority of our method to SVM and some other state-of-the-art methods, verifying that the margin distribution is more crucial for SVM-style learning approaches than minimum margin.

A preliminary version of this work appeared in a conference paper~\citep{Zhang2014Large}. Compared with the original version, the new approach has a simpler formulation and is more comprehensible. In addition, it avoids the operation of matrix inverse, so it can be more efficient when nonlinear kernels are applied. We also give a new theoretical analysis for the proposed algorithm, and present better empirical performance. The preliminary version was called LDM (large margin distribution machine), but it is not proper to call a better margin distribution as a ``larger'' one; thus we now call it ODM (optimal margin distribution learning machine), and the algorithm described in~\citep{Zhang2014Large} is called as ODM$^L$ in this paper.

The rest of this paper is organized as follows. Section~\ref{sec: preliminaries} introduces some preliminaries. Section~\ref{sec: Formulation} and~\ref{sec: optimization} present the formulation and optimization of ODM$^L$ and ODM respectively. Section~\ref{sec: analysis} present the theoretical analysis. Section~\ref{sec: experiments} reports on our experiments. Section~\ref{sec: related work} discusses about some related works. Finally, Section~\ref{sec: conclusions} concludes.

\section{Preliminaries} \label{sec: preliminaries}

We denote by $\Xcal \in \Rbb^d$ the instance space and $\Ycal =\{+1,-1\}$ the label set. Let $\Dcal$ be an unknown (underlying) distribution over $\Xcal \times \Ycal$. A training set of size $m$
\begin{align*}
S = \{(\xvec_1, y_1), (\xvec_2, y_2), \ldots, (\xvec_m, y_m)\},
\end{align*}
is drawn identically and independently (i.i.d.) according to the distribution $\Dcal$. Our goal is to learn a function which is used to predict the labels for future unseen instances.

For SVMs, $f$ is regarded as a linear model, i.e., $f(\xvec) = \wvec^\top \phi(\xvec)$ where $\wvec$ is a linear predictor, $\phi(\xvec)$ is a feature mapping of $\xvec$ induced by some kernel $k$, i.e., $k(\xvec_i, \xvec_j) = \phi(\xvec_i)^\top \phi(\xvec_j)$. According to~\citep{Cortes1995Support,Vapnik1995The}, the margin of instance $(\xvec_i, y_i)$ is formulated as
\begin{align} \label{defn: margin}
\gamma_i = y_i \wvec^\top \phi(\xvec_i), \ \forall i = 1,\ldots, m.
\end{align}

It can be found that in separable cases where the training examples can be separated with the zero error, all the $\gamma_i$ will be non-negative. By scaling it with $1 / \|\wvec\|$, we can get the geometric distance from $(\xvec_i, y_i)$ to $\wvec^\top \phi(\xvec)$, i.e.,
\begin{align*}
\hat{\gamma}_i = y_i \frac{\wvec^\top}{\|\wvec\|} \phi(\xvec_i), \ \forall i = 1, \ldots, m.
\end{align*}
From~\citep{Cristianini2000introduction}, it is shown that SVM with hard-margin (or Hard-margin SVM) is regarded as the maximization of the minimum distance,
\begin{align*}
\max_{\wvec} & \ \ \hat{\gamma} \\
\mbox{s.t.} & \ \ y_i \frac{\wvec^\top}{\|\wvec\|} \phi(\xvec_i) \geq \hat{\gamma}, \ i = 1, \ldots, m.
\end{align*}
It can be rewritten as
\begin{align*}
\max_{\wvec} & \ \ \frac{\gamma}{\|\wvec\|} \\
\mbox{s.t.} & \ \ y_i \wvec^\top \phi(\xvec_i) \geq \gamma, \ i = 1, \ldots, m.
\end{align*}
Since the value of $\gamma$ doesn't have influence on the optimization, we can simply set it as 1. Note that maximizing $1 / \|\wvec\|$ is equivalent to minimizing $\|\wvec\|^2/2$, we can get the classic formulation of Hard-margin SVM as follows:
\begin{align*}
\min_{\wvec} & \ \ \frac{1}{2} \wvec^\top \wvec \\
\mbox{s.t.} & \ \ y_i \wvec^{\top} \phi(\xvec_i) \geq 1, \ i = 1, \ldots, m.
\end{align*}

In non-separable cases where the training examples cannot be separated with the zero error, SVM with soft-margin (or Soft-margin SVM) is posed,
\begin{align} \label{eq: soft-margin SVM}
\begin{split}
\min_{\wvec, \xi_i} & \ \ \frac{1}{2} \wvec^\top \wvec + \frac{C}{m} \sum_{i=1}^{m} \xi_i \\
\mbox{s.t.} & \ \ y_i \wvec^{\top} \phi(\xvec_i) \geq 1 - \xi_i, \\
& \ \ \xi_i \geq 0, \ i = 1, \ldots, m.
\end{split}
\end{align}
where $\xi_i = [\xi_1, \ldots, \xi_m]^\top$ are slack variables which measure the losses of instances, and $C$ is a trading-off parameter. There exists a constant $\bar{C}$ such that (\ref{eq: soft-margin SVM}) can be equivalently reformulated as,
\begin{align*}
\max_{\wvec} & \ \ \gamma_0 - \frac{\bar{C}}{m} \sum\nolimits_{i=1}^{m} \xi_i \\
\mbox{s.t.} & \ \ \gamma_i \geq \gamma_0 - \xi_i,\\ & \ \ \xi_i \geq 0, \ i = 1, \ldots, m.
\end{align*}
where $\gamma_0$ is a relaxed minimum margin, and $\bar{C}$ is the trading-off parameter. Note that $\gamma_0$ indeed characterizes the top-$p$ minimum margin~\citep{Gao2012On}; hence, SVMs (with both hard-margin and soft-margin) consider only a single-point margin and have not exploited the whole margin distribution.

\section{Formulation} \label{sec: Formulation}

The two most straightforward statistics for characterizing the margin distribution are the first- and second-order statistics, that is, the mean and the variance of the margin. Formally, denote $\Xmat$ as the matrix whose $i$-th column is $\phi(\xvec_i)$, i.e., $\Xmat = [\phi(\xvec_1), \ldots, \phi(\xvec_m)]$, $\yvec = [y_1, \ldots, y_m]^\top$ is a column vector, and $\Ymat$ is a $m \times m$ diagonal matrix with $y_1, \ldots, y_m$ as the diagonal elements. According to the definition in (\ref{defn: margin}), the margin mean is
\begin{align} \label{eq: margin mean}
\bar{\gamma} = \frac{1}{m}\sum_{i=1}^{m} y_i \wvec^\top \phi(\xvec_i) = \frac{1}{m} (\Xmat \yvec)^\top \wvec,
\end{align}
and the margin variance is
\begin{align} \label{eq: margin variance}
\begin{split}
\hat{\gamma} & = \frac{1}{m} \sum_{i=1}^m (y_i \wvec^\top \phi(\xvec_i) - \bar{\gamma})^2    \\
& = \frac{1}{m} \wvec^\top \sum_{i=1}^m \phi(\xvec_i) \phi(\xvec_i)^\top \wvec - \frac{2}{m} \sum_{i=1}^m y_i \wvec^\top \phi(\xvec_i) \bar{\gamma} + \bar{\gamma}^2 \\
& = \frac{1}{m} \wvec^\top \Xmat \Xmat^\top \wvec - \frac{1}{m^2} \wvec^\top \Xmat \yvec \yvec^\top \Xmat^\top \wvec \\
& = \wvec^\top \Xmat \frac{m \Imat - \yvec \yvec^\top}{m^2} \Xmat^\top \wvec
\end{split}
\end{align}
where $\Imat$ is the identity matrix. Inspired by the recent theoretical result~\citep{Gao2012On}, we attempt to maximize the margin mean and minimize the margin variance simultaneously.

\subsection{ODM$^L$} \label{sec: LDM}

First consider a simpler way, i.e., adding the margin mean $\bar{\gamma}$ and the margin variance $\hat{\gamma}$ to the objective function of SVM explicitly. Then in the separable cases where the training examples can be separated with the zero error, the maximization of the margin mean and the minimization of the margin variance leads to the following hard-margin ODM$^L$,
\begin{align*}
\min_{\wvec} & \ \ \frac{1}{2} \wvec^\top \wvec + \lambda_1 \hat{\gamma} - \lambda_2 \bar{\gamma} \\
\mbox{s.t.} & \ \ y_i \wvec^{\top} \phi(\xvec_i)  \geq 1, \ i = 1, \ldots, m,
\end{align*}
where $\lambda_1$ and $\lambda_2$ are the parameters for trading-off the margin variance, the margin mean and the model complexity. It's evident that the hard-margin ODM$^L$ subsumes the hard-margin SVM when $\lambda_1$ and $\lambda_2$ equal $0$.

For the non-separable cases, similar to soft-margin SVM, the soft-margin ODM$^L$ leads to
\begin{align} \label{eq: ODM$^L$}
\begin{split}
\min_{\wvec, \xi_i} & \ \ \frac{1}{2} \wvec^\top \wvec + \lambda_1 \hat{\gamma} - \lambda_2 \bar{\gamma} + \frac{C}{m} \sum_{i=1}^{m} \xi_i \\
\mbox{s.t.}  & \ \ y_i \wvec^{\top} \phi(\xvec_i) \geq 1 - \xi_i,\\ & \ \ \xi_i \geq 0, \ i = 1, \ldots, m.
\end{split}
\end{align}
Similarly, soft-margin ODM$^L$ subsumes the soft-margin SVM if $\lambda_1$ and $\lambda_2$ both equal $0$. Because the soft-margin SVM often performs much better than the hard-margin one, in the following we will focus on soft-margin ODM$^L$ and if without clarification, ODM$^L$ is referred to the soft-margin ODM$^L$.

\subsection{ODM} \label{sec: ODM}

The idea of ODM$^L$ is quite straightforward, however, the final formulation is a little complex. In this section, we try to propose a simpler one.

Note that SVM set the minimum margin as 1 by scaling $\|\wvec\|$, following the similar way, we can also fix the margin mean as 1. Then the deviation of the margin of $(\xvec_i, y_i)$ to the margin mean is $|y_i \wvec^{\top} \phi(\xvec_i) - 1|$. By minimizing the margin variance, we arrive at the following formulation,
\begin{align} \label{eq: first ODM}
\begin{split}
\min_{\wvec, \xi_i, \epsilon_i} & \ \ \frac{1}{2} \wvec^\top \wvec + \frac{C}{m} \sum_{i=1}^m (\xi_i^2 + \epsilon_i^2) \\
\st & \ \ y_i \wvec^\top \phi(\xvec_i) \geq 1 - \xi_i, \\
    & \ \ y_i \wvec^\top \phi(\xvec_i) \leq 1 + \epsilon_i, \ i = 1, \ldots, m.
\end{split}
\end{align}
Since the margin of $(\xvec_i, y_i)$ is either smaller or greater than the margin mean, so at most one of $\xi_i$ and $\epsilon_i$ can be positive. In addition, if one is positive, the other must be zero (otherwise if it's negative, we can set it as zero without violating any constraint but decrease the objective function value), so the second term of the objective function is the margin variance.

The hyperplane $y_i \wvec^\top \phi(\xvec_i) = 1$ divides the space into two subspaces. For each example, no matter which space it lies in, it will suffer a loss which is quadratic with the deviation. However, the examples lie in the space corresponding to $y_i \wvec^\top \phi(\xvec_i) < 1$ are much easier to be misclassified than the other. So it is more reasonable to set different weights for the loss of examples in different spaces, i.e., the second term of (\ref{eq: first ODM}) can be modified as
\begin{align*}
\frac{1}{m} \sum_{i=1}^m (C_1 \xi_i^2 + C_2 \epsilon_i^2),
\end{align*}
where $C_1$ and $C_2$ are the trading-off parameters. According to the representer theorem~\citep{Scholkopf2001learning}, the optimal solution will be spanned by the support vectors. Unfortunately, for ODM, almost all training examples are support vectors. To make the solution sparse, we introduce a $D$-insensitive loss like SVR, i.e., the examples whose deviation is smaller than $D$ are tolerated and only those whose deviation is larger than $D$ will suffer a loss. Finally, we obtain the formulation of ODM,
\begin{align} \label{eq: ODM}
\begin{split}
\min_{\wvec, \xi_i, \epsilon_i} & \ \ \frac{1}{2} \wvec^\top \wvec + \frac{1}{m} \sum_{i=1}^m (C_1 \xi_i^2 + C_2 \epsilon_i^2) \\
\st & \ \ y_i \wvec^\top \phi(\xvec_i) \geq 1 - D - \xi_i, \\
    & \ \ y_i \wvec^\top \phi(\xvec_i) \leq 1 + D + \epsilon_i, \ i = 1, \ldots, m.
\end{split}
\end{align}
where $C_1$ and $C_2$ are described previously, $D$ is a parameter for controling the number of support vectors (sparsity of the solution).

\section{Optimization} \label{sec: optimization}

We first propose a dual coordinate descent method for kernel ODM$^L$ and ODM, and then propose a stochastic gradient descent with variance reduction for large scale linear kernel ODM$^L$ and ODM.

\subsection{$ODM_{dcd}$} \label{sec: ODM_dcd}

In this section we show that the dual of kernel ODM$^L$ and ODM are both convex quadratic optimization with only simple decoupled box constraints, and then present a dual coordinate descent method $ODM_{dcd}$ to solve them uniformly.

\subsubsection{Kernel ODM$^L$} \label{sec: kernelODM_0}
By substituting (\ref{eq: margin mean})-(\ref{eq: margin variance}), (\ref{eq: ODM$^L$}) leads to the following quadratic programming problem,
\begin{align} \label{eq: primal-ODM$^L$}
\begin{split}
\min_{\wvec, \xivec} & \ \ \frac{1}{2} \wvec^\top \wvec + \wvec^\top \Xmat \frac{\lambda_1(m \Imat - \yvec \yvec^\top)}{m^2} \Xmat^\top \wvec - \frac{\lambda_2}{m} (\Xmat \yvec)^\top \wvec + \frac{C}{m} \evec^\top \xivec \\
\mbox{s.t.} & \ \ \Ymat \Xmat^\top \wvec \geq \evec - \xivec, \\ & \ \ \xivec \geq \zerovec.
\end{split}
\end{align}
where $\evec$ stands for the all-one vector. Introduce the lagrange multipliers $\alphavec \geq \zerovec$ and $\betavec \geq \zerovec$ for the first and the second constraints respectively, the Lagrangian of (\ref{eq: primal-ODM$^L$}) leads to
\begin{align}
L(\wvec, \xivec, \alphavec, \betavec) & = \frac{1}{2} \wvec^\top \wvec + \wvec^\top \Xmat \frac{\lambda_1(m \Imat - \yvec \yvec^\top)}{m^2} \Xmat^\top \wvec - \frac{\lambda_2}{m} (\Xmat \yvec)^\top \wvec \notag \\
& \ \ \ + \frac{C}{m} \evec^\top \xivec - \alphavec^\top (\Ymat \Xmat^\top \wvec - \evec + \xivec) - \betavec^\top \xivec \notag \\
\label{eq: Lagrange ODM$^L$}
\begin{split}
& = \frac{1}{2} \wvec^\top \Qmat \wvec - \wvec^\top \Xmat \Ymat \left( \frac{\lambda_2}{m} \evec + \alphavec \right) + \alphavec^\top \evec \\
& \ \ \ + \xivec^\top \left( \frac{C}{m} \evec - \alphavec - \betavec \right)
\end{split}
\end{align}
where $\Qmat = \Imat + \Xmat \frac{2\lambda_1(m \Imat - \yvec \yvec^\top)}{m^2} \Xmat^\top$. By setting the partial derivations of $\{\wvec, \xivec\}$ to zero, we have
\begin{align}
\label{eq: KKT1 ODM$^L$}
\frac{\partial L}{\partial \wvec} & = \Qmat \wvec - \Xmat \Ymat \left( \frac{\lambda_2}{m} \evec + \alphavec \right) \Longrightarrow \wvec = \Qmat^{-1} \Xmat \Ymat \left( \frac{\lambda_2}{m} \evec + \alphavec \right), \\
\label{eq: KKT2 ODM$^L$}
\frac{\partial L}{\partial \xivec} & = \frac{C}{m} \evec - \alphavec - \betavec \Longrightarrow \zerovec \leq \alphavec \leq \frac{C}{m} \evec.
\end{align}
By substituting (\ref{eq: KKT1 ODM$^L$}) and (\ref{eq: KKT2 ODM$^L$}) into (\ref{eq: Lagrange ODM$^L$}), the dual of (\ref{eq: primal-ODM$^L$}) can be cast as:
\begin{align} \label{eq: dual ODM$^L$ temp}
\begin{split}
\min_{\alphavec} & \ \ f(\alphavec) = \frac{1}{2} \left( \frac{\lambda_2}{m} \evec + \alphavec \right)^\top \Ymat \Xmat^\top \Qmat^{-1} \Xmat \Ymat \left( \frac{\lambda_2}{m} \evec + \alphavec \right) - \evec^\top \alphavec \\
\mbox{s.t.} & \ \ \zerovec \leq \alphavec \leq \frac{C}{m} \evec.
\end{split}
\end{align}
Note that the dimension of $\Qmat$ depends on the feature mapping $\phi(\xvec)$ and may not be calculated if $\phi(\xvec)$ maps an instance into an infinite dimension space. Fortunately, $\Xmat^\top \Qmat^{-1} \Xmat$ is a $m \times m$ square matrix, next we show how to calculate this matrix.

\begin{lemma}
For any matrix $\Xmat$ and $\Amat$, it holds that $(\Imat + \Xmat \Amat \Xmat^\top)^{-1} = \Imat - \Xmat (\Amat^{-1} + \Xmat^\top \Xmat)^{-1} \Xmat^\top$.
\end{lemma}

\begin{proof}
By multiplying the right side with $\Imat + \Xmat \Amat \Xmat^\top$, we have
\begin{align*}
& \ \ \ \ (\Imat - \Xmat (\Amat^{-1} + \Xmat^\top \Xmat)^{-1} \Xmat^\top)(\Imat + \Xmat \Amat \Xmat^\top) \\
& = \Imat - \Xmat (\Amat^{-1} + \Xmat^\top \Xmat)^{-1} \Xmat^\top + \Xmat \Amat \Xmat^\top - \Xmat (\Amat^{-1} + \Xmat^\top \Xmat)^{-1} \Xmat^\top \Xmat \Amat \Xmat^\top \\
& = \Imat - \Xmat (\Amat^{-1} + \Xmat^\top \Xmat)^{-1} \Xmat^\top + \Xmat \Amat \Xmat^\top \\
& \ \ \ - \Xmat (\Amat^{-1} + \Xmat^\top \Xmat)^{-1} (\Xmat^\top \Xmat + \Amat^{-1} - \Amat^{-1}) \Amat \Xmat^\top \\
& = \Imat - \Xmat (\Amat^{-1} + \Xmat^\top \Xmat)^{-1} \Xmat^\top + \Xmat \Amat \Xmat^\top - \Xmat \Amat \Xmat^\top \\
& \ \ \ + \Xmat (\Amat^{-1} + \Xmat^\top \Xmat)^{-1} \Xmat^\top \\
& = \Imat
\end{align*}
It is shown that
\begin{align*}
(\Imat + \Xmat \Amat \Xmat^\top)^{-1} = \Imat - \Xmat (\Amat^{-1} + \Xmat^\top \Xmat)^{-1} \Xmat^\top.
\end{align*}
\end{proof}

According to Lemma 1, we have
\begin{align} \label{eq: invQ x}
\begin{split}
\Qmat^{-1} \Xmat & = (\Imat - \Xmat (\Amat^{-1} + \Xmat^\top \Xmat)^{-1} \Xmat^\top) \Xmat \\
& = \Xmat (\Imat - (\Amat^{-1} + \Gmat)^{-1} \Gmat) \\
& = \Xmat (\Imat - (\Amat^{-1} + \Gmat)^{-1} (\Amat^{-1} + \Gmat - \Amat^{-1})) \\
& = \Xmat (\Imat - \Imat + (\Amat^{-1} + \Gmat)^{-1} \Amat^{-1}) \\
& = \Xmat (\Imat + \Amat \Gmat)^{-1}
\end{split}
\end{align}
where $\Amat = \frac{2\lambda_1(m \Imat - \yvec \yvec^\top)}{m^2}$ and $\Gmat = \Xmat^\top \Xmat$ is the kernel matrix. Then
\begin{align*}
\Xmat^\top \Qmat^{-1} \Xmat = \Gmat (\Imat + \Amat \Gmat)^{-1}.
\end{align*}
By denoting $\Hmat = \Ymat \Gmat (\Imat + \Amat \Gmat)^{-1} \Ymat$, the objective function of (\ref{eq: dual ODM$^L$ temp}) can be written as
\begin{align*}
f(\alphavec) & = \frac{1}{2} \left( \frac{\lambda_2}{m} \evec + \alphavec \right)^\top \Hmat \left( \frac{\lambda_2}{m} \evec + \alphavec \right) - \evec^\top \alphavec \\
& = \frac{1}{2} \alphavec^\top \Hmat \alphavec + \frac{\lambda_2}{m} \evec^\top \Hmat \alphavec - \evec^\top \alphavec + const \\
& = \frac{1}{2} \alphavec^\top \Hmat \alphavec + \left( \frac{\lambda_2}{m} \Hmat \evec - \evec \right)^\top \alphavec + const
\end{align*}
Negelect the const term which doesn't have influence on the optimization, we arrive at the final formulation of the dual kernel ODM$^L$,
\begin{align} \label{eq: dual ODM$^L$}
\begin{split}
\min_{\alphavec} & \ \ \frac{1}{2} \alphavec^\top \Hmat \alphavec + \left( \frac{\lambda_2}{m} \Hmat \evec - \evec \right)^\top \alphavec \\
\mbox{s.t.} & \ \ \zerovec \leq \alphavec \leq \frac{C}{m} \evec.
\end{split}
\end{align}

For prediction, according to (\ref{eq: KKT1 ODM$^L$}) and (\ref{eq: invQ x}), one can obtain the coefficients $\wvec$ from the optimal $\alphavec^*$ as
\begin{align*}
\wvec = \Xmat (\Imat + \Amat \Gmat)^{-1} \Ymat \left( \frac{\lambda_2}{m} \evec + \alphavec^* \right) = \Xmat \thetavec,
\end{align*}
where $\thetavec = (\Imat + \Amat \Gmat)^{-1} \Ymat (\frac{\lambda_2}{m} \evec + \alphavec^*)$. Hence for testing instance $\zvec$, its label can be obtained by
\begin{align*}
sgn \left( \wvec^\top \phi(\zvec) \right) = sgn \left( \sum_{i=1}^{m} \theta_i k(\xvec_i, \zvec) \right).
\end{align*}

\subsubsection{Kernel ODM} \label{sec: kernelODM}

Introduce the lagrange multipliers $\zetavec \geq \zerovec$ and $\betavec \geq \zerovec$ for the two constraints respectively, the Lagrangian of (\ref{eq: ODM}) leads to
\begin{align} \label{eq: Lagrange ODM}
\begin{split}
L(\wvec, \xivec, \epsilonvec, \zetavec, \betavec) & = \frac{1}{2} \wvec^\top \wvec + \frac{C_1}{m} \xivec^\top \xivec + \frac{C_2}{m} \epsilonvec^\top \epsilonvec - \zetavec^\top (\Ymat \Xmat^\top \wvec - (1 - D) \evec + \xivec) \\
& \ \ \ + \betavec^\top (\Ymat \Xmat^\top \wvec - (1 + D) \evec - \xivec).
\end{split}
\end{align}
By setting the partial derivations of $\wvec, \xivec, \epsilonvec$ to zero, we have
\begin{align}
\label{eq: KKT1 ODM}
\frac{\partial L}{\partial \wvec} & = \wvec - \Xmat \Ymat \zetavec + \Xmat \Ymat \betavec \Longrightarrow \wvec = \Xmat \Ymat (\zetavec - \betavec) \\
\label{eq: KKT2 ODM}
\frac{\partial L}{\partial \xivec} & = \frac{2C_1}{m} \xivec - \zetavec \Longrightarrow \xivec = \frac{m}{2C_1} \zetavec \\
\label{eq: KKT3 ODM}
\frac{\partial L}{\partial \epsilonvec} & = \frac{2C_2}{m} \epsilonvec - \betavec \Longrightarrow \epsilonvec = \frac{m}{2C_2} \betavec
\end{align}
By substituting (\ref{eq: KKT1 ODM}), (\ref{eq: KKT2 ODM}) and (\ref{eq: KKT3 ODM}) into (\ref{eq: Lagrange ODM}), we have
\begin{align*}
L(\zetavec, \betavec) & = - \frac{1}{2} (\zetavec - \betavec)^\top \Ymat \Xmat^\top \Xmat \Ymat (\zetavec - \betavec) + \frac{C_1}{m} \frac{m^2}{4 C_1^2} \zetavec^\top \zetavec + \frac{C_2}{m} \frac{m^2}{4 C_2^2} \betavec^\top \betavec \\
& \ \ \ - \frac{m}{2C_1} \zetavec^\top \zetavec - \frac{m}{2C_2} \betavec^\top \betavec + (1 - D) \zetavec^\top \evec - (1 + D) \betavec^\top \evec \\
& = - \frac{1}{2} (\zetavec - \betavec)^\top \Qmat (\zetavec - \betavec) - \frac{m}{4C_1} \zetavec^\top \zetavec - \frac{m}{4C_2} \betavec^\top \betavec \\
& \ \ \ + (1 - D) \zetavec^\top \evec - (1 + D) \betavec^\top \evec
\end{align*}
where $\Qmat = \Ymat \Xmat^\top \Xmat \Ymat$. Denote $\alphavec^\top = [\zetavec^\top, \betavec^\top]$, then $\zetavec = [\Imat, \zerovec] \alphavec$, $\betavec = [\zerovec, \Imat] \alphavec$ and $\zetavec - \betavec = [\Imat, -\Imat] \alphavec$. The Lagrangian can be rewritten as
\begin{align*}
L(\zetavec, \betavec) & = - \frac{1}{2} \alphavec^\top [\Imat, -\Imat]^\top \Qmat [\Imat, -\Imat] \alphavec - \frac{m}{4C_1} \alphavec^\top [\Imat, \zerovec]^\top [\Imat, \zerovec] \alphavec \\
& \ \ \ - \frac{m}{4C_2} \alphavec^\top [\zerovec, \Imat]^\top [\zerovec, \Imat] \alphavec + (1 - D) \evec^\top [\Imat, \zerovec] \alphavec - (1 + D) \evec^\top [\zerovec, \Imat] \alphavec \\
& = - \frac{1}{2} \alphavec^\top
\begin{bmatrix}
\Qmat & -\Qmat \\
- \Qmat & \Qmat
\end{bmatrix}
 \alphavec - \frac{1}{2} \alphavec^\top
\begin{bmatrix}
\frac{m}{2C_1} \Imat & \zerovec \\
\zerovec & \frac{m}{2C_2} \Imat
\end{bmatrix}
\alphavec + \begin{bmatrix}
(1 - D) \evec \\
- (1 + D) \evec
\end{bmatrix}^\top
\alphavec
\end{align*}
where $\evec$ stands for the all-one vector. Thus the dual of (\ref{eq: ODM}) can be cast as:
\begin{align} \label{eq: final ODM}
\begin{split}
\min_{\alphavec} & \ \ \ \frac{1}{2} \alphavec^\top \begin{bmatrix}
\Qmat + \frac{m}{2C_1} \Imat & -\Qmat \\
- \Qmat & \Qmat + \frac{m}{2C_2} \Imat
\end{bmatrix} \alphavec + \begin{bmatrix}
(D - 1) \evec \\
(D + 1) \evec
\end{bmatrix}^\top \alphavec \\
\st & \ \ \ \alphavec \geq \zerovec.
\end{split}
\end{align}

For prediction, according to (\ref{eq: KKT1 ODM}), one can obtain the coefficients $\wvec$ from the optimal $\alphavec^*$ as
\begin{align*}
\wvec = \Xmat \Ymat (\zetavec - \betavec) = \Xmat \Ymat [\Imat, -\Imat] \alphavec^* = \Xmat \thetavec,
\end{align*}
where $\thetavec = \Ymat [\Imat, -\Imat] \alphavec^*$. Hence for testing instance $\zvec$, its label can be obtained by
\begin{align*}
sgn \left( \wvec^\top \phi(\zvec) \right) = sgn \left( \sum_{i=1}^{m} \theta_i k(\xvec_i, \zvec) \right).
\end{align*}

\subsubsection{Dual Coordinate Descent} \label{sec: DCD}

Note that (\ref{eq: dual ODM$^L$}) and (\ref{eq: final ODM}) are both the special cases of the following form, which has convex quadratic objective function and simple decoupled box constraints,
\begin{align*}
\min_{\alphavec} & \ \ \ f(\alphavec) = \frac{1}{2} \alphavec^\top \Hmat \alphavec + \qvec^\top \alphavec \\
\st & \ \ \ \zerovec \leq \alphavec \leq \uvec.
\end{align*}
where $\uvec = \boldsymbol{\infty}$ for ODM. As suggested by~\citep{Yuan2012Recent}, it can be efficiently solved by the dual coordinate descent method. In dual coordinate descent method~\citep{Hsieh2008A}, one of the variables is selected to minimize while the other variables are kept as constants at each iteration, and a closed-form solution can be achieved at each iteration. Specifically, to minimize $\alpha_i$ by keeping the other $\alpha_{j\neq i}$'s as constants, one needs to solve the following subproblem,
\begin{align} \label{eq: simpleQP}
\begin{split}
\min_{t} & \ \ f(\alphavec + t \evec_i) \\
\mbox{s.t.} & \ \ 0 \leq \alpha_i + t \leq u_i,
\end{split}
\end{align}
where $\evec_i$ denotes the vector with $1$ in the $i$-th coordinate and $0$'s elsewhere. Let $\Hmat = [h_{ij}]_{i,j=1,\dots,m}$, we have
\begin{align*}
f(\alphavec + t \evec_i) = \frac{1}{2} h_{ii} t^2 + [\nabla f(\alphavec)]_i t + f(\alphavec),
\end{align*}
where $[\nabla f(\alphavec)]_i$ is the $i$-th component of the gradient $\nabla f(\alphavec)$. Note that $f(\alphavec)$ is independent of $t$ and thus can be dropped. Considering that $f(\alphavec + t \evec_i)$ is a simple quadratic function of $t$, and further note the box constraint $0 \leq \alpha_i \leq u_i$, the minimizer of (\ref{eq: simpleQP}) leads to a closed-form solution,
\begin{align*}
\alpha_i^{new} = \min\left( \max \left( \alpha_i - \frac{[\nabla f(\alphavec)]_i}{h_{ii}}, 0 \right), u_i \right).
\end{align*}
Algorithm~\ref{alg: ODM_dcd} summarizes the pseudo-code of $ODM_{dcd}$ for kernel ODM$^L$ and ODM.

\begin{algorithm}[h]
\caption{$ODM_{dcd}$}
\begin{algorithmic}
\STATE {\bfseries Input:} Data set $\Xmat$.
\STATE {\bfseries Output:} $\alphavec$.
\STATE  Initialize $\alphavec = \zerovec$, calculate $\Hmat$ and $\qvec$.
\WHILE{$\alphavec$ not converge}
\FOR{$i = 1, \dots m$}
\STATE  $[\nabla f(\alphavec)]_i \leftarrow [\Hmat \alphavec + \qvec]_i$.
\STATE  $\alpha_i \leftarrow \min\left( \max \left( \alpha_i -\frac{[\nabla f(\alphavec)]_i}{h_{ii}}, 0 \right), u_i \right)$.
\ENDFOR
\ENDWHILE
\end{algorithmic}
\label{alg: ODM_dcd}
\end{algorithm}

\subsection{$ODM_{svrg}$} \label{sec: svrg}

In section~\ref{sec: DCD}, the proposed method can efficiently deal with kernel ODM$^L$ and ODM. However, the inherent computational cost for the kernel matrix takes $O(m^2)$ time, which might be computational prohibitive for large scale problems. To make them more useful, in the following, we present a fast linear kernel ODM$^L$ and ODM for large scale problems by adopting the stochastic gradient descent with variance reduction (SVRG)~\citep{Polyak1992Acceleration, Accelerating2013Johnson}.

For linear kernel ODM$^L$, (\ref{eq: ODM$^L$}) can be reformulated as the following form,
\begin{align} \label{eq: linear ODM$^L$}
\begin{split}
\min_{\wvec} \ \ f_L(\wvec) & = \frac{1}{2} \wvec^\top \wvec + \frac{\lambda_1}{m} \wvec^\top \Xmat \Xmat^\top \wvec - \frac{\lambda_1}{m^2} \wvec^\top \Xmat \yvec \yvec^\top \Xmat^\top \wvec \\
& \ \ \ - \frac{\lambda_2}{m} (\Xmat \yvec)^\top \wvec + \frac{C}{m} \sum_{i=1}^m \max\{0, 1 - y_i \wvec^\top \xvec_i\},
\end{split}
\end{align}
where $\Xmat = [\xvec_1, \ldots, \xvec_m]$, $\yvec = [y_1, \ldots, y_m]^\top$ is a column vector. For linear kernel ODM, (\ref{eq: ODM}) can be reformulated as the following form,
\begin{align} \label{eq: linear ODM}
\begin{split}
\min_{\wvec} \ \ f_O(\wvec) & = \frac{1}{2} \wvec^\top \wvec + \frac{C_1}{m} \sum_{i=1}^m \max \{0, 1 - D - y_i \wvec^\top \xvec_i\}^2 \\
& \ \ \ + \frac{C_2}{m} \sum_{i=1}^m \max \{0, y_i \wvec^\top \xvec_i - 1 - D\}^2.
\end{split}
\end{align}

For large scale problems, computing the gradient of (\ref{eq: linear ODM$^L$}) and (\ref{eq: linear ODM}) is expensive because its computation involves all the training examples. Stochastic gradient descent (SGD) works by computing a noisy unbiased estimation of the gradient via sampling a subset of the training examples. Theoretically, when the objective is convex, it can be shown that in expectation, SGD converges to the global optimal solution~\citep{Kushner2003Stochastic,Bottou2010Large-Scale}. During the past decade, SGD has been applied to various machine learning problems and achieved promising performances~\citep{Zhang2004Solving,Shalev-Shwartz2007Pegasos,Bordes2009SGD-QN,Shamir2013Stochastic,Accelerating2013Johnson,Reddi2015On,Zhao2015Stochastic}.

The following theorem presents an approach to obtain an unbiased estimation of the gradient $\nabla f_L(\wvec)$ and $\nabla f_O(\wvec)$.

\begin{theorem} \label{thm: SGD for ODM$^L$ODM}
If two examples $(\xvec_i, y_i)$ and $(\xvec_j, y_j)$ are randomly sampled from the training set independently, then
\begin{align} \label{eq: SGD update for ODM$^L$}
\begin{split}
\nabla f_L(\wvec, \xvec_i, \xvec_j) & = \wvec + 2 \lambda_1 \xvec_i \xvec_i^\top \wvec - 2 \lambda_1 y_i y_j \xvec_i \xvec_j ^\top \wvec - \lambda_2 y_i \xvec_i \\
& \ \ \ - C y_i \xvec_i \mathbb{I}(i \in I_1)
\end{split}
\end{align}
and
\begin{align} \label{eq: SGD update for ODM}
\begin{split}
\nabla f_O(\wvec, \xvec_i) & = \wvec + 2 C_1 (y_i \wvec^\top \xvec_i + D - 1) y_i \xvec_i \mathbb{I}(i \in I_2) \\
& \ \ \ + 2 C_2 (y_i \wvec^\top \xvec_i - D - 1) y_i \xvec_i \mathbb{I}(i \in I_3)
\end{split}
\end{align}
are the unbiased estimation of $\nabla f_L(\wvec)$ and $\nabla f_O(\wvec)$ respectively, where $\mathbb{I}(\cdot)$ is the indicator function that returns 1 when the argument holds, and 0 otherwise. $I_1$, $I_2$, $I_3$ are the index sets defined as
\begin{align*}
I_1 & \equiv \{ i \ | \ y_i \wvec^\top \xvec_i < 1 \}, \\
I_2 & \equiv \{ i \ | \ y_i \wvec^\top \xvec_i < 1 - D \}, \\
I_3 & \equiv \{ i \ | \ y_i \wvec^\top \xvec_i > 1 + D \}.
\end{align*}
\end{theorem}

\begin{proof}
Note that the gradient of $f_L(\wvec)$ is
\begin{align*}
\nabla f_L(\wvec) & = \wvec + \frac{2 \lambda_1}{m} \Xmat \Xmat^\top \wvec - \frac{2 \lambda_1}{m^2} \Xmat \yvec \yvec^\top \Xmat^\top \wvec - \frac{\lambda_2}{m} \Xmat \yvec \\
 & \ \ \ - \frac{C}{m} \sum_{i=1}^m y_i \xvec_i \mathbb{I}(i \in I_1).
\end{align*}
Further note that
\begin{align} \label{eq: expectation}
E [\xvec_i \xvec_i^\top] = \frac{1}{m} \sum_{i=1}^m \xvec_i \xvec_i^\top = \frac{1}{m} \Xmat \Xmat^\top, \ \ E [y_i \xvec_i] = \frac{1}{m} \sum_{i=1}^m y_i \xvec_i = \frac{1}{m} \Xmat \yvec.
\end{align}
According to the linearity of expectation, the independence between $\xvec_i$ and $\xvec_j$, and with (\ref{eq: expectation}), we have
\begin{align*}
E [\nabla f_L(\wvec, \xvec_i, \xvec_j)] & = \wvec + 2 \lambda_1 E [\xvec_i \xvec_i^\top] \wvec - 2 \lambda_1 E [y_i \xvec_i] E [y_j \xvec_j]^\top \wvec \\
& \ \ \ - \lambda_2 E [y_i \xvec_i] - C E [y_i \xvec_i \mathbb{I}(i \in I_1)] \\
& = \wvec + \frac{2 \lambda_1}{m} \Xmat \Xmat^\top \wvec - \frac{2 \lambda_1}{m^2} \Xmat \yvec \yvec^\top \Xmat^\top \wvec \\
& \ \ \ - \frac{\lambda_2}{m} \Xmat \yvec - \frac{C}{m} \sum_{i=1}^m y_i \xvec_i \mathbb{I}(i \in I_1) \\
& = \nabla f_L(\wvec)
\end{align*}
It is shown that $\nabla f_L(\wvec, \xvec_i, \xvec_j)$ is a noisy unbiased gradient of $f_L(\wvec)$.

Again the gradient of $f_O(\wvec)$ is
\begin{align*}
\nabla f_O(\wvec) & = \wvec + \frac{2 C_1}{m} \sum_{i=1}^m (y_i \wvec^\top \xvec_i + D - 1) y_i \xvec_i \mathbb{I}(i \in I_2) \\
& \ \ \ + \frac{2 C_2}{m} \sum_{i=1}^m (y_i \wvec^\top \xvec_i - D - 1) y_i \xvec_i \mathbb{I}(i \in I_3).
\end{align*}
According to the linearity of expectation, and with (\ref{eq: expectation}), we have
\begin{align*}
E [\nabla f_O(\wvec, \xvec_i)] & = \wvec + 2 C_1 \Ebb[(y_i \wvec^\top \xvec_i + D - 1) y_i \xvec_i \mathbb{I}(i \in I_2)] \\
& \ \ \ + 2 C_2 \Ebb[(y_i \wvec^\top \xvec_i - D - 1) y_i \xvec_i \mathbb{I}(i \in I_3)] \\
& = \wvec + \frac{2 C_1}{m} \sum_{i=1}^m (y_i \wvec^\top \xvec_i + D - 1) y_i \xvec_i \mathbb{I}(i \in I_2) \\
& \ \ \ + \frac{2 C_2}{m} \sum_{i=1}^m (y_i \wvec^\top \xvec_i - D - 1) y_i \xvec_i \mathbb{I}(i \in I_3) \\
& = \nabla f_O(\wvec)
\end{align*}
It is shown that $\nabla f_O(\wvec, \xvec_i)$ is a noisy unbiased gradient of $f_O(\wvec)$.
\end{proof}

With Theorem~\ref{thm: SGD for ODM$^L$ODM}, the stochastic gradient update can be formed as
\begin{align} \label{eq:SGD update}
\wvec_{t+1} = \wvec_t - \eta_t \gvec_t,
\end{align}
where $\gvec_t = \nabla f_L(\wvec_t, \xvec_i, \xvec_j)$ for ODM$^L$ and $g_t = \nabla f_O(\wvec_t, \xvec_i)$ for ODM, $\eta_t$ is a suitably chosen step-size parameter in the $t$-th iteration.

Since the objective function of ODM is differentiable, in practice we use the stochastic gradient descent with variance reduction (SVRG) which is more robust than SGD~\citep{Accelerating2013Johnson}. Besides performing the normal stochastic gradient update (\ref{eq:SGD update}) at each iteration, it also occasionally compute full gradient, which can be used to reduce the variance of the stochastic gradient estimation.

Algorithm~\ref{alg: ODM_{svrg}} summarizes the pseudo-code of $ODM_{svrg}$.

\begin{algorithm}[h]
\caption{$ODM_{svrg}$}
\begin{algorithmic}
\STATE {\bfseries Input:} Data set $\Xmat$.
\STATE {\bfseries Output:} $\wvecbar$
\STATE  Initialize $\wvecbar_0 = \zerovec$.
\FOR{$s = 1, 2, \ldots$}
\STATE  $\wvecbar = \wvecbar_{s-1}$.
\STATE  Compute full gradient $\muvecbar$
\STATE  $\wvec_0 = \wvecbar$
\FOR{$t = 1, 2, \ldots, m$}
\STATE  Randomly sample training example $(\xvec_i, y_i)$.
\STATE  Compute $\gvec_t$ as in (\ref{eq: SGD update for ODM}).
\STATE  $\wvec_t = \wvec_{t-1} - \eta (\nabla f_O(\wvec_{t-1}, \xvec_i) - \nabla f_O(\wvecbar, \xvec_i) + \muvecbar)$.
\ENDFOR
\STATE  Set $\wvecbar_s = \wvec_t$ for randomly chosen $t \in \{1, 2, \ldots, m\}$.
\ENDFOR
\end{algorithmic}
\label{alg: ODM_{svrg}}
\end{algorithm}

\section{Analysis} \label{sec: analysis}

In this section, we study the statistical property of ODM$^L$ and ODM. Here we only consider the linear case for simplicity, however, the results are also applicable to any other feature mapping $\phi$. As indicated in Section~\ref{sec: ODM_dcd}, the dual problem both take the following form,
\begin{align} \label{eq: dual linear problem}
\begin{split}
\min_{\alphavec} & \ \ f(\alphavec) = \frac{1}{2} \alphavec^\top \Hmat \alphavec + \qvec^\top \alphavec, \\
\mbox{s.t.} & \ \ 0 \leq \alpha_i \leq u_i, \ i = 1, \ldots, m.
\end{split}
\end{align}

\begin{lemma} \label{lem: bound for qp}
Let $\alphavec$ denote the optimal solution of (\ref{eq: dual linear problem}), and suppose
\begin{align} \label{eq: solution}
\begin{split}
\alphavec^* & = \argmin_{\zerovec \leq \alphavec \leq \uvec} f(\alphavec), \\
\alphavec^i & = \argmin_{\zerovec \leq \alphavec \leq \uvec, \alpha_i = 0} f(\alphavec), \ \ i = 1, \ldots, m,
\end{split}
\end{align}
then we have
\begin{align*}
\frac{[\Hmat \alphavec^i + \qvec]_i^2}{2 h_{ii}} \leq \frac{{\alpha_i^*}^2}{2} h_{ii} - \alpha_i^* [\Hmat \alphavec^* + \qvec]_i,
\end{align*}
where $[\cdot]_i$ denotes the $i$-th component of a vector and $h_{ii}$ is the $(i,i)$-th entry of the matrix $\Hmat$.
\end{lemma}

\begin{proof}
According to the definition in (\ref{eq: solution}), we have
\begin{align}
\label{eq: relation}
f(\alphavec^i) - \min_t f(\alphavec^i + t \evec_i) \leq f(\alphavec^i) - f(\alphavec^*) \leq f(\alphavec^* - \alpha_i^* \evec_i) - f(\alphavec^*),
\end{align}
where $\evec_i$ denotes a vector with $1$ in the $i$-th coordinate and $0$'s elsewhere.

Note that
\begin{align*}
f(\alphavec^i) - \min_t f(\alphavec^i + t \evec_i) & = f(\alphavec^i) - \min_t \left\{ f(\alphavec^i) + \frac{t^2}{2} h_{ii} + t {\alphavec^i}^\top \Hmat \evec_i + t \qvec^\top \evec_i \right\} \\
& = - \min_t \left\{ \frac{t^2}{2} h_{ii} + t (\Hmat \alphavec^i + \qvec)^\top \evec_i \right\} \\
& = \frac{[\Hmat \alphavec^i + \qvec]_i^2}{2 h_{ii}}
\end{align*}
and
\begin{align*}
f(\alphavec^* - \alpha_i^* \evec_i) - f(\alphavec^*) = \frac{{\alpha_i^*}^2}{2} h_{ii} - \alpha_i^* (\Hmat \alphavec^* + \qvec)^\top \evec_i = \frac{{\alpha_i^*}^2}{2} h_{ii} - \alpha_i^* [\Hmat \alphavec^* + \qvec]_i
\end{align*}
combine with (\ref{eq: relation}), it is shown that
\begin{align*}
\frac{[\Hmat \alphavec^i + \qvec]_i^2}{2 h_{ii}} \leq \frac{{\alpha_i^*}^2}{2} h_{ii} - \alpha_i^* [\Hmat \alphavec^* + \qvec]_i.
\end{align*}
\end{proof}

Based Lemma~\ref{lem: bound for qp}, we derive the following two bounds for ODM$^L$ and ODM on the expectation of error according to the leave-one-out cross-validation estimate, which is an unbiased estimate of the probability of test error. As shown in~\citep{Luntz1969On},
\begin{align} \label{eq: leave-one-out}
\Ebb[R(\alphavec)] = \frac{\Ebb[\Lcal((\xvec_1,y_1), \ldots, (\xvec_m, y_m))]}{m},
\end{align}
where $\Lcal((\xvec_1,y_1), \ldots, (\xvec_m, y_m))$ is the number of errors in the leave-one-out procedure.

\begin{theorem} \label{thm: bound ODM$^L$}
Let $\alphavec$ denote the optimal solution of the dual problem of ODM$^L$, then we have
\begin{align} \label{eq: bound ODM$^L$}
\Ebb[R(\alphavec)] \leq \frac{\Ebb[\sum_{i \in I_1} \alpha_i^* h_{ii} + |I_2|]}{m},
\end{align}
where $I_1 \equiv \{ i \ | \ 0 < \alpha_i^* < C/m \}$, $I_2 \equiv \{ i \ | \ \alpha_i^* = C/m \}$.
\end{theorem}

\begin{proof}
According to the derivation in Section~\ref{sec: kernelODM_0}, for ODM$^L$ we have
\begin{align*}
[\Hmat \alphavec + \qvec]_i & = \left[ \Hmat \alphavec + \frac{\lambda_2}{m} \Hmat \evec - \evec \right]_i \\
& = \left[ \Ymat \Xmat^\top \Xmat (\Imat + \Amat \Gmat)^{-1} \Ymat \left( \alphavec + \frac{\lambda_2}{m} \evec \right) - \evec \right]_i,
\end{align*}
further note that
\begin{align*}
\wvec = \Xmat (\Imat + \Amat \Gmat)^{-1} \Ymat \left( \alphavec + \frac{\lambda_2}{m} \evec \right),
\end{align*}
so it is shown that
\begin{align*}
[\Hmat \alphavec + \qvec]_i = [\Ymat \Xmat^\top \wvec - \evec]_i = y_i \xvec_i^\top \wvec - 1.
\end{align*}

Suppose the corresponding solution of $\alphavec^*$ and $\alphavec^i$ for the primal problem of ODM$^L$ are $\wvec^*$ and $\wvec^i$, respectively. According to Lemma~\ref{lem: bound for qp} we have
\begin{align*}
\frac{(y_i \xvec_i^\top \wvec^i - 1)^2}{2 h_{ii}} \leq \frac{{\alpha_i^*}^2}{2} h_{ii} - \alpha_i^* (y_i \xvec_i^\top \wvec^* - 1).
\end{align*}

1) $\alpha_i^* = 0$. The right-hand side equals 0, which indicates that the left-hand side must equal 0, i.e., all these examples will be correctly classified by $\wvec^i$.

2) $0 < \alpha_i^* < C/m$. According to the complementary slackness conditions, in this case we have $y_i \xvec_i^\top \wvec^* = 1$. For any misclassified example $(\xvec_i, y_i)$, i.e., $y_i \xvec_i^\top \wvec^i < 0$, we have $1 \leq \alpha_i^* h_{ii}$.

3) $\alpha_i^* = C/m$. All these examples may be misclassified in the leave-one-out procedure.

So we have
\begin{align*}
\Lcal((\xvec_1,y_1), \ldots, (\xvec_m, y_m)) \leq \sum_{i \in I_1} \alpha_i^* h_{ii} + |I_2|,
\end{align*}
where $I_1 \equiv \{ i \ | \ 0 < \alpha_i^* < C/m \}$, $I_2 \equiv \{ i \ | \ \alpha_i^* = C/m \}$. Take expectation on both side and with (\ref{eq: leave-one-out}), we get that (\ref{eq: bound ODM$^L$}) holds.
\end{proof}

\begin{theorem} \label{thm: bound ODM}
Let $\alphavec$ denote the optimal solution of the dual problem of ODM, then we have
\begin{align} \label{eq: bound ODM}
\Ebb[R(\alphavec)] \leq \frac{\Ebb[\sum_{i \in I_1} \alpha_i^* \left(\| \xvec_i \|^2 + \frac{m}{2C1}\right) + \sum_{i \in I_2} \alpha_i^* \left(\| \xvec_i \|^2 + \frac{m}{2C2}\right) + D(|I_1| - |I_2|)]}{m},
\end{align}
where $I_1 \equiv \{ i \ | \ \alpha_i^* > 0 \ and \ 1 \leq i \leq m \}$, $I_2 \equiv \{ i \ | \ \alpha_i^* > 0 \ and \ m+1 \leq i \leq 2m \}$.
\end{theorem}

\begin{proof}
Denote $\alphavec^\top = [\zetavec^\top, \betavec^\top]$, according to the derivation in Section~\ref{sec: kernelODM} we have
\begin{align*}
[\Hmat \alphavec + \qvec]_i & =
\begin{bmatrix}
\Qmat (\zetavec - \betavec) + \frac{2C_1}{m} \zetavec + (D - 1) \evec \\
- \Qmat (\zetavec - \betavec) + \frac{2C_2}{m} \betavec + (D + 1) \evec
\end{bmatrix}_i \\ & = \begin{bmatrix}
\Ymat \Xmat^\top \Xmat \Ymat (\zetavec - \betavec) + \frac{2C_1}{m} \zetavec + (D - 1) \evec \\
- \Ymat \Xmat^\top \Xmat \Ymat (\zetavec - \betavec) + \frac{2C_2}{m} \betavec + (D + 1) \evec
\end{bmatrix}_i,
\end{align*}
further note that $\wvec = \Xmat \Ymat (\zetavec - \betavec)$ and the definition in (\ref{eq: solution}), so it is shown that
\begin{align*}
[\Hmat \alphavec^i + \qvec]_i =
\begin{cases}
y_i \xvec_i^\top \wvec^i - (1 - D), & 1 \leq i \leq m. \\
- y_i \xvec_i^\top \wvec^i + (1 + D), & m+1 \leq i \leq 2m. \\
\end{cases}
\end{align*}
and
\begin{align*}
[\Hmat \alphavec^* + \qvec]_i =
\begin{cases}
y_i \xvec_i^\top \wvec^* + \xi_i - (1 - D), & 1 \leq i \leq m. \\
- y_i \xvec_i^\top \wvec^* + \epsilon_i + (1 + D), & m+1 \leq i \leq 2m. \\
\end{cases}
\end{align*}
according to Lemma~\ref{lem: bound for qp} we have
\begin{align*}
\frac{((1 - D) - y_i \xvec_i^\top \wvec^i)^2}{2 (\| \xvec_i \|^2 + \frac{m}{2C1})} & \leq \frac{{\alpha_i^*}^2}{2} \left(\| \xvec_i \|^2 + \frac{m}{2C1}\right) \\
& \ \ \ - \alpha_i^* (y_i \xvec_i^\top \wvec^* + \xi_i - (1 - D)), \ \ 1 \leq i \leq m. \\
\frac{((1 + D) - y_i \xvec_i^\top \wvec^i)^2}{2 (\| \xvec_i \|^2 + \frac{m}{2C2})} & \leq \frac{{\alpha_i^*}^2}{2} \left(\| \xvec_i \|^2 + \frac{m}{2C2}\right) \\
& \ \ \ - \alpha_i^* (- y_i \xvec_i^\top \wvec^* + \epsilon_i + (1 + D)), \ \ m+1 \leq i \leq 2m.
\end{align*}

1) $\alpha_i^* = 0$. The right-hand side equals 0, which indicates that the left-hand side must equal 0, i.e., all these examples will be correctly classified by $\wvec^i$.

2) $\alpha_i^* > 0$. According to the complementary slackness conditions, in this case the second term of the right-hand side must equal 0. For any misclassified example $(\xvec_i, y_i)$, i.e., $y_i \xvec_i^\top \wvec^i < 0$, we have
\begin{align*}
1 \leq \alpha_i^* \left(\| \xvec_i \|^2 + \frac{m}{2C1}\right) + D, & \ \ 1 \leq i \leq m. \\
1 \leq \alpha_i^* \left(\| \xvec_i \|^2 + \frac{m}{2C2}\right) - D, & \ \ m+1 \leq i \leq 2m.
\end{align*}

So we have
\begin{align*}
\Lcal((\xvec_1,y_1), \ldots, (\xvec_m, y_m)) & \leq \sum_{i \in I_1} \alpha_i^* \left(\| \xvec_i \|^2 + \frac{m}{2C1}\right) \\
& \ \ \ + \sum_{i \in I_2} \alpha_i^* \left(\| \xvec_i \|^2 + \frac{m}{2C2}\right) + D(|I_1| - |I_2|),
\end{align*}
where $I_1 \equiv \{ i \ | \ \alpha_i^* > 0 \ and \ 1 \leq i \leq m \}$, $I_2 \equiv \{ i \ | \ \alpha_i^* > 0 \ and \ m+1 \leq i \leq 2m \}$. Take expectation on both side and with (\ref{eq: leave-one-out}), we get that (\ref{eq: bound ODM}) holds.
\end{proof}

\section{Empirical Study} \label{sec: experiments}

In this section, we empirically evaluate the effectiveness of our methods on a broad range of data sets. We first introduce the experimental settings in Section~\ref{sec: experimental setup}, and then compare ODM$^L$ and ODM with SVM and Linear Discriminant Analysis (LDA) in Section~\ref{sec: results regular} and Section~\ref{sec: results large}. In addition, we also study the cumulative margin distribution produced by ODM$^L$, ODM and SVM in Section~\ref{sec: plot margin distribution}. The computational cost is presented in Section~\ref{sec: plot time}.

\begin{table}[!htb]
\scriptsize
\caption{Characteristics of experimental data sets.}
\vskip 0.1in
\centering
\scalebox{1.05}[1.05]{
\begin{tabular}{ c | l c c | l c c}
\hline \hline \noalign{\smallskip}
Scale & Dataset & \#Instance & \#Feature & Dataset & \#Instance & \#Feature \\
\noalign{\smallskip} \hline \noalign{\smallskip}
\emph{regular} & \emph{colon-cancer} & {\myfont 62} & {\myfont 2,000} & \emph{cylinder-bands} & {\myfont 277} & {\myfont 39} \\
& \emph{echocardiogram} & {\myfont 62} & {\myfont 8} & \emph{heart-c} & {\myfont 296} & {\myfont 13} \\
& \emph{balloons} & {\myfont 76} & {\myfont 4} & \emph{haberman} & {\myfont 306} & {\myfont 14} \\
& \emph{hepatitis} & {\myfont 80} & {\myfont 19} & \emph{liverDisorders} & {\myfont 345} & {\myfont 6} \\
& \emph{promoters} & {\myfont 106} & {\myfont 57} & \emph{house-votes} & {\myfont 435} & {\myfont 16} \\
& \emph{planning} & {\myfont 182} & {\myfont 12} & \emph{vehicle} & {\myfont 435} & {\myfont 16} \\
& \emph{colic} & {\myfont 188} & {\myfont 13} & \emph{clean1} & {\myfont 476} & {\myfont 166} \\
& \emph{parkinsons} & {\myfont 195} & {\myfont 22} & \emph{wdbc} & {\myfont 569} & {\myfont 14} \\
& \emph{colic.ORIG} & {\myfont 205} & {\myfont 17} & \emph{isolet} & {\myfont 600} & {\myfont 51} \\
& \emph{sonar} & {\myfont 208} & {\myfont 60} & \emph{credit-a} & {\myfont 653} & {\myfont 15} \\
& \emph{house} & {\myfont 232} & {\myfont 16} & \emph{austra} & {\myfont 690} & {\myfont 15} \\
& \emph{vote} & {\myfont 232} & {\myfont 16} & \emph{australian} & {\myfont 690} & {\myfont 42} \\
& \emph{heart-h} & {\myfont 261} & {\myfont 10} & \emph{diabetes} & {\myfont 768} & {\myfont 8} \\
& \emph{heart} & {\myfont 270} & {\myfont 9} & \emph{fourclass} & {\myfont 862} & {\myfont 2} \\
& \emph{heart-statlog} & {\myfont 270} & {\myfont 13} & \emph{credit-g} & {\myfont 1,000} & {\myfont 20} \\
& \emph{breast} & {\myfont 277} & {\myfont 9} & \emph{german} & {\myfont 1,000} & {\myfont 59} \\
\noalign{\smallskip} \hline \noalign{\smallskip}
\emph{large} & \emph{farm-ads} & {\myfont 4,143} & {\myfont 54,877} & \emph{real-sim} & {\myfont 72,309} & {\myfont 20,958} \\
 & \emph{news20} & {\myfont 19,996} & {\myfont 1,355,191} & \emph{mini-boo-ne} & {\myfont 130,064} & {\myfont 50} \\
 & \emph{adult-a} & {\myfont 32,561} & {\myfont 123} & \emph{ijcnn1} & {\myfont 141,691} & {\myfont 22} \\
 & \emph{w8a} & {\myfont 49,749} & {\myfont 300} & \emph{rcv1} & {\myfont 697,641} & {\myfont 47,236} \\
 & \emph{cod-rna} & {\myfont 59,535} & {\myfont 8} & \emph{kdd2010} & {\myfont 8,407,752} & {\myfont 20,216,830} \\
\noalign{\smallskip} \hline \hline
\end{tabular}}
\label{table data set info}
\end{table}

\begin{table}[!htb]
\scriptsize
\caption{Accuracy (mean$\pm$std.) comparison on regular scale data sets. Linear kernel is used. The best accuracy on each data set is bolded. $\bullet$/$\circ$ indicates the performance is significantly better/worse than SVM (paired $t$-tests at 95\% significance level). The win/tie/loss counts are summarized in the last row.}
\centering
\vskip 0.1in
\scalebox{1.2}[1.2]{
\begin{tabular}{ l | l l l l }
\hline \hline \noalign{\smallskip}
Dataset & \multicolumn{1}{c}{SVM} & \multicolumn{1}{c}{LDA} & \multicolumn{1}{c}{ODM$^L$} & \multicolumn{1}{c}{ODM} \\
\noalign{\smallskip} \hline \noalign{\smallskip}
\emph{colon-cancer} & {\myfont 0.808$\pm$0.070} & {\myfont 0.640$\pm$0.136}$\circ$ & {\myfont 0.806$\pm$0.070} & {\myfontb 0.823$\pm$0.059}$\bullet$ \\
\emph{echocardiogram} & {\myfont 0.663$\pm$0.069} & {\myfontb 0.705$\pm$0.082}$\bullet$ & {\myfont 0.704$\pm$0.061}$\bullet$ & {\myfontb 0.705$\pm$0.055}$\bullet$ \\
\emph{balloons} & {\myfont 0.703$\pm$0.050} & {\myfont 0.719$\pm$0.067} & {\myfont 0.694$\pm$0.046} & {\myfontb 0.720$\pm$0.058}$\bullet$ \\
\emph{hepatitis} & {\myfont 0.808$\pm$0.055} & {\myfont 0.751$\pm$0.052}$\circ$ & {\myfont 0.821$\pm$0.058}$\bullet$ & {\myfontb 0.848$\pm$0.040}$\bullet$ \\
\emph{promoters} & {\myfont 0.723$\pm$0.071} & {\myfont 0.595$\pm$0.068}$\circ$ & {\myfont 0.721$\pm$0.069} & {\myfontb 0.736$\pm$0.067} \\
\emph{planning} & {\myfont 0.683$\pm$0.031} & {\myfont 0.463$\pm$0.052}$\circ$ & {\myfontb 0.706$\pm$0.034}$\bullet$ & {\myfontb 0.706$\pm$0.034}$\bullet$ \\
\emph{colic} & {\myfont 0.814$\pm$0.035} & {\myfont 0.766$\pm$0.029}$\circ$ & {\myfont 0.832$\pm$0.026}$\bullet$ & {\myfontb 0.839$\pm$0.035}$\bullet$ \\
\emph{parkinsons} & {\myfont 0.846$\pm$0.038} & {\myfont 0.718$\pm$0.041}$\circ$ & {\myfontb 0.865$\pm$0.030}$\bullet$ & {\myfont 0.858$\pm$0.027}$\bullet$ \\
\emph{colic.ORIG} & {\myfont 0.618$\pm$0.027} & {\myfont 0.550$\pm$0.054}$\circ$ & {\myfont 0.619$\pm$0.042} & {\myfontb 0.633$\pm$0.033}$\bullet$ \\
\emph{sonar} & {\myfont 0.725$\pm$0.039} & {\myfont 0.613$\pm$0.055}$\circ$ & {\myfont 0.736$\pm$0.036} & {\myfontb 0.741$\pm$0.035}$\bullet$ \\
\emph{house} & {\myfont 0.942$\pm$0.015} & {\myfont 0.845$\pm$0.026}$\circ$ & {\myfontb 0.968$\pm$0.011}$\bullet$ & {\myfont 0.965$\pm$0.010}$\bullet$ \\
\emph{vote} & {\myfont 0.934$\pm$0.022} & {\myfont 0.847$\pm$0.023}$\circ$ & {\myfontb 0.970$\pm$0.014}$\bullet$ & {\myfont 0.968$\pm$0.013}$\bullet$ \\
\emph{heart-h} & {\myfont 0.807$\pm$0.028} & {\myfont 0.800$\pm$0.032} & {\myfont 0.803$\pm$0.030} & {\myfontb 0.812$\pm$0.023}$\bullet$ \\
\emph{heart} & {\myfont 0.799$\pm$0.029} & {\myfontb 0.809$\pm$0.032} & {\myfont 0.791$\pm$0.030} & {\myfont 0.801$\pm$0.027} \\
\emph{heart-statlog} & {\myfont 0.826$\pm$0.031} & {\myfont 0.792$\pm$0.024}$\circ$ & {\myfont 0.821$\pm$0.022} & {\myfontb 0.836$\pm$0.028}$\bullet$ \\
\emph{breast} & {\myfont 0.717$\pm$0.033} & {\myfont 0.703$\pm$0.026}$\circ$ & {\myfont 0.725$\pm$0.027}$\bullet$ & {\myfontb 0.732$\pm$0.035}$\bullet$ \\
\emph{cylinder-bands} & {\myfont 0.684$\pm$0.039} & {\myfont 0.634$\pm$0.030}$\circ$ & {\myfontb 0.708$\pm$0.038}$\bullet$ & {\myfontb 0.708$\pm$0.033}$\bullet$ \\
\emph{heart-c} & {\myfont 0.795$\pm$0.029} & {\myfont 0.759$\pm$0.032}$\circ$ & {\myfont 0.800$\pm$0.026} & {\myfontb 0.806$\pm$0.030}$\bullet$ \\
\emph{haberman} & {\myfont 0.734$\pm$0.030} & {\myfont 0.568$\pm$0.076}$\circ$ & {\myfontb 0.738$\pm$0.020} & {\myfont 0.734$\pm$0.018} \\
\emph{liverDisorders} & {\myfont 0.675$\pm$0.030} & {\myfont 0.515$\pm$0.035}$\circ$ & {\myfontb 0.681$\pm$0.026} & {\myfont 0.672$\pm$0.030} \\
\emph{house-votes} & {\myfont 0.935$\pm$0.012} & {\myfont 0.838$\pm$0.014}$\circ$ & {\myfontb 0.942$\pm$0.010}$\bullet$ & {\myfontb 0.942$\pm$0.013}$\bullet$ \\
\emph{vehicle} & {\myfont 0.959$\pm$0.012} & {\myfont 0.738$\pm$0.033}$\circ$ & {\myfont 0.959$\pm$0.013} & {\myfontb 0.961$\pm$0.011} \\
\emph{clean1} & {\myfont 0.803$\pm$0.035} & {\myfont 0.539$\pm$0.036}$\circ$ & {\myfont 0.814$\pm$0.019}$\bullet$ & {\myfontb 0.825$\pm$0.024}$\bullet$ \\
\emph{wdbc} & {\myfont 0.963$\pm$0.012} & {\myfont 0.887$\pm$0.022}$\circ$ & {\myfont 0.968$\pm$0.011}$\bullet$ & {\myfontb 0.969$\pm$0.011}$\bullet$ \\
\emph{isolet} & {\myfont 0.995$\pm$0.003} & {\myfont 0.935$\pm$0.026}$\circ$ & {\myfont 0.997$\pm$0.002}$\bullet$ & {\myfontb 0.998$\pm$0.003}$\bullet$ \\
\emph{credit-a} & {\myfont 0.861$\pm$0.014} & {\myfont 0.783$\pm$0.030}$\circ$ & {\myfontb 0.864$\pm$0.013}$\bullet$ & {\myfont 0.862$\pm$0.014} \\
\emph{austra} & {\myfont 0.857$\pm$0.013} & {\myfont 0.667$\pm$0.071}$\circ$ & {\myfont 0.859$\pm$0.015} & {\myfontb 0.862$\pm$0.013}$\bullet$ \\
\emph{australian} & {\myfont 0.844$\pm$0.019} & {\myfont 0.693$\pm$0.058}$\circ$ & {\myfont 0.866$\pm$0.014}$\bullet$ & {\myfontb 0.867$\pm$0.014}$\bullet$ \\
\emph{diabetes} & {\myfont 0.671$\pm$0.019} & {\myfontb 0.701$\pm$0.020}$\bullet$ & {\myfont 0.670$\pm$0.021} & {\myfont 0.671$\pm$0.020} \\
\emph{fourclass} & {\myfontb 0.724$\pm$0.014} & {\myfont 0.723$\pm$0.015} & {\myfont 0.723$\pm$0.014} & {\myfont 0.723$\pm$0.014} \\
\emph{credit-g} & {\myfont 0.726$\pm$0.041} & {\myfont 0.678$\pm$0.015}$\circ$ & {\myfont 0.745$\pm$0.015}$\bullet$ & {\myfontb 0.747$\pm$0.016}$\bullet$ \\
\emph{german} & {\myfont 0.711$\pm$0.030} & {\myfont 0.673$\pm$0.015}$\circ$ & {\myfont 0.738$\pm$0.016}$\bullet$ & {\myfontb 0.743$\pm$0.015}$\bullet$ \\
\noalign{\smallskip} \hline \noalign{\smallskip}
Avg. & \multicolumn{1}{c}{0.792} & \multicolumn{1}{c}{0.708} & \multicolumn{1}{c}{0.802} & \multicolumn{1}{c}{0.807} \\\noalign{\smallskip} \hline \noalign{\smallskip}
SVM: w/t/l & & \multicolumn{1}{c}{26/4/2} & \multicolumn{1}{c}{0/15/17} & \multicolumn{1}{c}{0/8/24} \\
\noalign{\smallskip} \hline \hline
\end{tabular}}
\label{table linear regular}
\end{table}

\begin{table}[!htb]
\scriptsize
\caption{Accuracy (mean$\pm$std.) comparison on regular scale data sets. RBF kernel is used. The best accuracy on each data set is bolded. $\bullet$/$\circ$ indicates the performance is significantly better/worse than SVM (paired $t$-tests at 95\% significance level). The win/tie/loss counts are summarized in the last row.}
\centering
\vskip 0.1in
\scalebox{1.2}[1.2]{
\begin{tabular}{ l | l l l l }
\hline \hline \noalign{\smallskip}
Dataset & \multicolumn{1}{c}{SVM} & \multicolumn{1}{c}{LDA} & \multicolumn{1}{c}{ODM$^L$} & \multicolumn{1}{c}{ODM} \\
\noalign{\smallskip} \hline \noalign{\smallskip}
\emph{colon-cancer} & {\myfont 0.710$\pm$0.105} & {\myfont 0.619$\pm$0.140}$\circ$ & {\myfont 0.725$\pm$0.090} & {\myfontb 0.744$\pm$0.074}$\bullet$ \\
\emph{echocardiogram} & {\myfont 0.703$\pm$0.062} & {\myfontb 0.740$\pm$0.063}$\bullet$ & {\myfont 0.717$\pm$0.067} & {\myfont 0.731$\pm$0.066}$\bullet$ \\
\emph{balloons} & {\myfont 0.699$\pm$0.051} & {\myfont 0.702$\pm$0.079} & {\myfont 0.739$\pm$0.059}$\bullet$ & {\myfontb 0.746$\pm$0.045}$\bullet$ \\
\emph{hepatitis} & {\myfont 0.819$\pm$0.030} & {\myfont 0.705$\pm$0.069}$\circ$ & {\myfont 0.830$\pm$0.032}$\bullet$ & {\myfontb 0.837$\pm$0.033}$\bullet$ \\
\emph{promoters} & {\myfont 0.684$\pm$0.100} & {\myfont 0.623$\pm$0.065}$\circ$ & {\myfont 0.715$\pm$0.074}$\bullet$ & {\myfontb 0.742$\pm$0.065}$\bullet$ \\
\emph{planning} & {\myfont 0.708$\pm$0.035} & {\myfont 0.540$\pm$0.191}$\circ$ & {\myfont 0.707$\pm$0.034} & {\myfontb 0.710$\pm$0.032} \\
\emph{colic} & {\myfont 0.822$\pm$0.033} & {\myfont 0.826$\pm$0.022} & {\myfontb 0.841$\pm$0.018}$\bullet$ & {\myfontb 0.841$\pm$0.022}$\bullet$ \\
\emph{parkinsons} & {\myfont 0.929$\pm$0.029} & {\myfont 0.735$\pm$0.126}$\circ$ & {\myfont 0.927$\pm$0.029} & {\myfontb 0.933$\pm$0.026} \\
\emph{colic.ORIG} & {\myfont 0.638$\pm$0.043} & {\myfont 0.581$\pm$0.072}$\circ$ & {\myfont 0.641$\pm$0.044} & {\myfontb 0.647$\pm$0.040} \\
\emph{sonar} & {\myfont 0.842$\pm$0.034} & {\myfont 0.526$\pm$0.041}$\circ$ & {\myfont 0.846$\pm$0.032} & {\myfontb 0.857$\pm$0.029}$\bullet$ \\
\emph{house} & {\myfont 0.953$\pm$0.020} & {\myfont 0.874$\pm$0.022}$\circ$ & {\myfontb 0.964$\pm$0.013}$\bullet$ & {\myfont 0.961$\pm$0.015}$\bullet$ \\
\emph{vote} & {\myfont 0.946$\pm$0.016} & {\myfont 0.876$\pm$0.019}$\circ$ & {\myfontb 0.968$\pm$0.013}$\bullet$ & {\myfont 0.964$\pm$0.011}$\bullet$ \\
\emph{heart-h} & {\myfont 0.801$\pm$0.031} & {\myfont 0.804$\pm$0.030} & {\myfont 0.801$\pm$0.029} & {\myfontb 0.811$\pm$0.026}$\bullet$ \\
\emph{heart} & {\myfont 0.808$\pm$0.025} & {\myfont 0.786$\pm$0.037}$\circ$ & {\myfont 0.822$\pm$0.029}$\bullet$ & {\myfontb 0.838$\pm$0.022}$\bullet$ \\
\emph{heart-statlog} & {\myfont 0.815$\pm$0.027} & {\myfont 0.793$\pm$0.024}$\circ$ & {\myfont 0.829$\pm$0.026}$\bullet$ & {\myfontb 0.834$\pm$0.024}$\bullet$ \\
\emph{breast} & {\myfont 0.729$\pm$0.030} & {\myfont 0.706$\pm$0.031}$\circ$ & {\myfont 0.753$\pm$0.027}$\bullet$ & {\myfontb 0.755$\pm$0.027}$\bullet$ \\
\emph{cylinder-bands} & {\myfont 0.738$\pm$0.040} & {\myfont 0.632$\pm$0.058}$\circ$ & {\myfont 0.736$\pm$0.042} & {\myfontb 0.753$\pm$0.033}$\bullet$ \\
\emph{heart-c} & {\myfont 0.788$\pm$0.028} & {\myfont 0.722$\pm$0.026}$\circ$ & {\myfont 0.801$\pm$0.021}$\bullet$ & {\myfontb 0.808$\pm$0.027}$\bullet$ \\
\emph{haberman} & {\myfont 0.727$\pm$0.024} & {\myfont 0.528$\pm$0.103}$\circ$ & {\myfont 0.731$\pm$0.027} & {\myfontb 0.742$\pm$0.021}$\bullet$ \\
\emph{liverDisorders} & {\myfont 0.719$\pm$0.030} & {\myfont 0.563$\pm$0.036}$\circ$ & {\myfont 0.712$\pm$0.031} & {\myfontb 0.721$\pm$0.032} \\
\emph{house-votes} & {\myfont 0.945$\pm$0.013} & {\myfont 0.854$\pm$0.020}$\circ$ & {\myfont 0.949$\pm$0.012}$\bullet$ & {\myfontb 0.950$\pm$0.011}$\bullet$ \\
\emph{vehicle} & {\myfont 0.992$\pm$0.007} & {\myfont 0.692$\pm$0.023}$\circ$ & {\myfont 0.993$\pm$0.006} & {\myfontb 0.994$\pm$0.006} \\
\emph{clean1} & {\myfont 0.890$\pm$0.020} & {\myfont 0.671$\pm$0.060}$\circ$ & {\myfontb 0.891$\pm$0.024} & {\myfont 0.889$\pm$0.023} \\
\emph{wdbc} & {\myfont 0.951$\pm$0.011} & {\myfont 0.702$\pm$0.019}$\circ$ & {\myfont 0.961$\pm$0.010}$\bullet$ & {\myfontb 0.974$\pm$0.010}$\bullet$ \\
\emph{isolet} & {\myfontb 0.998$\pm$0.002} & {\myfont 0.892$\pm$0.072}$\circ$ & {\myfontb 0.998$\pm$0.002} & {\myfontb 0.998$\pm$0.002} \\
\emph{credit-a} & {\myfont 0.858$\pm$0.014} & {\myfont 0.797$\pm$0.022}$\circ$ & {\myfont 0.861$\pm$0.013} & {\myfontb 0.864$\pm$0.012}$\bullet$ \\
\emph{austra} & {\myfont 0.853$\pm$0.013} & {\myfont 0.583$\pm$0.055}$\circ$ & {\myfont 0.857$\pm$0.014}$\bullet$ & {\myfontb 0.862$\pm$0.015}$\bullet$ \\
\emph{australian} & {\myfont 0.815$\pm$0.014} & {\myfont 0.695$\pm$0.069}$\circ$ & {\myfont 0.854$\pm$0.016}$\bullet$ & {\myfontb 0.868$\pm$0.012}$\bullet$ \\
\emph{diabetes} & {\myfont 0.773$\pm$0.014} & {\myfont 0.686$\pm$0.015}$\circ$ & {\myfont 0.771$\pm$0.014} & {\myfontb 0.777$\pm$0.014}$\bullet$ \\
\emph{fourclass} & {\myfont 0.998$\pm$0.003} & {\myfont 0.717$\pm$0.013}$\circ$ & {\myfont 0.998$\pm$0.003} & {\myfontb 1.000$\pm$0.002}$\bullet$ \\
\emph{credit-g} & {\myfont 0.751$\pm$0.014} & {\myfont 0.625$\pm$0.107}$\circ$ & {\myfont 0.750$\pm$0.016} & {\myfontb 0.757$\pm$0.016}$\bullet$ \\
\emph{german} & {\myfont 0.731$\pm$0.019} & {\myfont 0.599$\pm$0.018}$\circ$ & {\myfont 0.743$\pm$0.016}$\bullet$ & {\myfontb 0.750$\pm$0.011}$\bullet$ \\
\noalign{\smallskip} \hline \noalign{\smallskip}
Avg. & \multicolumn{1}{c}{0.817} & \multicolumn{1}{c}{0.700} & \multicolumn{1}{c}{0.826} & \multicolumn{1}{c}{0.833} \\
\noalign{\smallskip} \hline \noalign{\smallskip}
SVM: w/t/l & & \multicolumn{1}{c}{28/3/1} & \multicolumn{1}{c}{0/17/15} & \multicolumn{1}{c}{0/7/25} \\
\noalign{\smallskip} \hline \hline
\end{tabular}}
\label{table kernel regular}
\end{table}

\begin{table}[!htb]
\scriptsize
\caption{Accuracy (mean$\pm$std.) comparison on large scale data sets. Linear kernel is used. The best accuracy on each data set is bolded. $\bullet$/$\circ$ indicates the performance is significantly better/worse than SVM (paired $t$-tests at 95\% significance level). The win/tie/loss counts are summarized in the last row. LDA did not return results on some data sets in 48 hours.}
\centering
\vskip 0.1in
\scalebox{1.2}[1.2]{
\begin{tabular}{ l | l l l l }
\hline \hline \noalign{\smallskip}
Dataset & \multicolumn{1}{c}{SVM} & \multicolumn{1}{c}{LDA} & \multicolumn{1}{c}{ODM$^L$} & \multicolumn{1}{c}{ODM} \\
\noalign{\smallskip} \hline \noalign{\smallskip}
\emph{farm-ads} & {\myfont 0.880$\pm$0.007} & \multicolumn{1}{c}{N/A} & {\myfont 0.890$\pm$0.008}$\bullet$ & {\myfontb 0.892$\pm$0.006}$\bullet$ \\
\emph{news20} & {\myfont 0.954$\pm$0.002} & \multicolumn{1}{c}{N/A} & {\myfontb 0.960$\pm$0.001}$\bullet$ & {\myfont 0.956$\pm$0.001}$\bullet$ \\
\emph{adult-a} & {\myfont 0.845$\pm$0.002} & {\myfont 0.719$\pm$0.003}$\circ$ & {\myfontb 0.846$\pm$0.003}$\bullet$ & {\myfontb 0.846$\pm$0.002}$\bullet$ \\
\emph{w8a} & {\myfontb 0.983$\pm$0.001} & {\myfont 0.510$\pm$0.006}$\circ$ & {\myfontb 0.983$\pm$0.001} & {\myfont 0.982$\pm$0.001} \\
\emph{cod-rna} & {\myfontb 0.899$\pm$0.001} & {\myfont 0.503$\pm$0.001}$\circ$ & {\myfontb 0.899$\pm$0.001} & {\myfontb 0.899$\pm$0.001} \\
\emph{real-sim} & {\myfont 0.961$\pm$0.001} & \multicolumn{1}{c}{N/A} & {\myfont 0.971$\pm$0.001}$\bullet$ & {\myfontb 0.972$\pm$0.001}$\bullet$ \\
\emph{mini-boo-ne} & {\myfont 0.855$\pm$0.005} & \multicolumn{1}{c}{N/A} & {\myfont 0.848$\pm$0.001}$\circ$ & {\myfontb 0.895$\pm$0.004}$\bullet$ \\
\emph{ijcnn1} & {\myfontb 0.921$\pm$0.003} & \multicolumn{1}{c}{N/A} & {\myfontb 0.921$\pm$0.002} & {\myfont 0.919$\pm$0.001} \\
\emph{rcv1} & {\myfont 0.969$\pm$0.000} & \multicolumn{1}{c}{N/A} & {\myfontb 0.977$\pm$0.000}$\bullet$ & {\myfontb 0.977$\pm$0.000}$\bullet$ \\
\emph{kdda2010} & {\myfont 0.852$\pm$0.000} & \multicolumn{1}{c}{N/A} & {\myfontb 0.881$\pm$0.001}$\bullet$ & {\myfont 0.880$\pm$0.001}$\bullet$ \\
\noalign{\smallskip} \hline \noalign{\smallskip}
Avg. & \multicolumn{1}{c}{0.912} & \multicolumn{1}{c}{0.577} & \multicolumn{1}{c}{0.918} & \multicolumn{1}{c}{0.922} \\\noalign{\smallskip} \hline \noalign{\smallskip}
SVM: w/t/l & & \multicolumn{1}{c}{0/0/3} & \multicolumn{1}{c}{1/3/6} & \multicolumn{1}{c}{0/3/7} \\
\noalign{\smallskip} \hline \hline
\end{tabular}}
\label{table linear large}
\end{table}

\subsection{Experimental Setup} \label{sec: experimental setup}

We evaluate the effectiveness of our proposed methods on thirty two regular scale data sets and ten large scale data sets, including both UCI data sets and real-world data sets like KDD2010\footnote{https://pslcdatashop.web.cmu.edu/KDDCup/downloads.jsp}. Table~\ref{table data set info} summarizes the statistics of these data sets. The data set size is ranged from 62 to more than 8,000,000, and the dimensionality is ranged from 2 to more than 20,000,000, covering a broad range of properties. All features are normalized into the interval $[0, 1]$. For each data set, half of examples are randomly selected as the training data, and the remaining examples are used as the testing data. For regular scale data sets, both linear and RBF kernels are evaluated. Experiments are repeated for 30 times with random data partitions, and the average accuracies as well as the standard deviations are recorded. For large scale data sets, linear kernel is evaluated. Experiments are repeated for 10 times with random data partitions, and the average accuracies (with standard deviations) are recorded.

ODM$^L$ and ODM are compared with standard SVM which ignores the margin distribution, and Linear Discriminant Analysis (LDA)~\citep{Fisher1936The}. For SVM and ODM$^L$, the regularization parameter $C$ is selected by 5-fold cross validation from $[10, 50, 100]$. In addition, the regularization parameters $\lambda_1$, $\lambda_2$ are selected from the set of $[2^{-8}, \ldots, 2^{-2}]$. For ODM, the regularization parameter $C_1$ and $C_2$ are selected from the set of $[2^0, \ldots, 2^{10}]$, while the parameter $D$ is selected from the set of $[0, 0.1, \ldots, 0.5]$. The parameters $\eta$ used for $ODM_{svrg}$ are set with the same setup in~\citep{Accelerating2013Johnson}. The width of the RBF kernel for SVM, LDA, ODM$^L$ and ODM are selected by 5-fold cross validation from the set of $[2^{-2}\delta, \ldots, 2^2\delta]$, where $\delta$ is the average distance between instances. All selections are performed on training sets.

\subsection{Results on Regular Scale Data Sets} \label{sec: results regular}

Tables~\ref{table linear regular} and~\ref{table kernel regular} summarize the results on thirty two regular scale data sets. As can be seen, the overall performance of our methods are superior or highly competitive to SVM. Specifically, for linear kernel, ODM$^L$/ODM performs significantly better than SVM on 17/24 over 32 data sets, respectively, and achieves the best accuracy on 31 data sets; for RBF kernel, ODM$^L$/ODM performs significantly better than SVM on 15/25 over 32 data sets, respectively, and achieves the best accuracy on 31 data sets. In addition, as can be seen, in comparing with standard SVM which does not consider margin distribution, the win/tie/loss counts show that ODM$^L$ and ODM are always better or comparable, never worse than SVM.

\subsection{Results on Large Scale Data Sets} \label{sec: results large}

Table~\ref{table linear large} summarizes the results on ten large scale data sets. LDA did not return results on some data sets due to the high computational cost. As can be seen, the overall performance of our methods are superior or highly competitive to SVM. Specifically, ODM$^L$/ODM performs significantly better than SVM on 6/7 over 10 data sets, respectively, and achieves the best accuracy on almost all data sets.

\subsection{Margin Distributions} \label{sec: plot margin distribution}

\begin{figure}[!htb]
\begin{center}
\centerline{\includegraphics[width=\columnwidth]{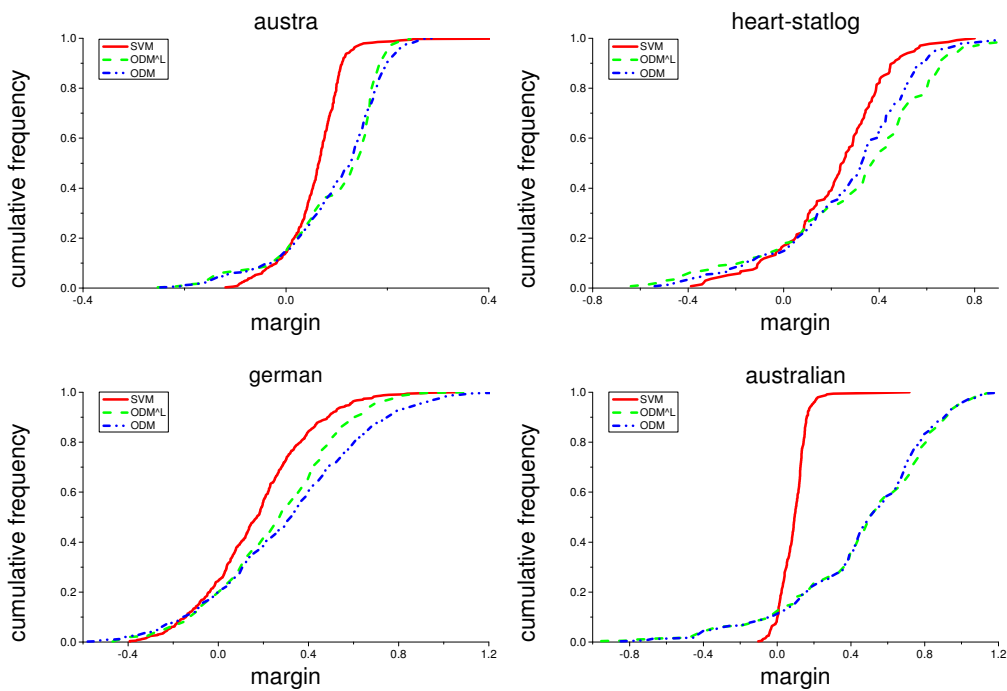}}
\caption{Cumulative frequency ($y$-axis) with respect to margin ($x$-axis) of SVM, ODM$^L$ and ODM on some representative regular scale data sets. The more right the curve, the larger the accumulated margin.}
\label{fig: plot margin distributions}
\end{center}
\end{figure}

Figure~\ref{fig: plot margin distributions} plots the cumulative margin distribution of SVM, ODM$^L$ and ODM on some representative regular scale data sets. The curves for other data sets are more or less similar. The point where a curve and the $x$-axis crosses is the corresponding minimum margin. As can be seen, our methods usually have a little bit smaller minimum margin than SVM, whereas the curve of ODM$^L$ and ODM generally lies on the right side, showing that the margin distribution of ODM$^L$ and ODM are generally better than that of SVM. In other words, for most examples, our methods generally produce a larger margin than SVM.

\subsection{Time Cost} \label{sec: plot time}

We compare the time cost of our methods and SVM on the ten large scale data sets. All the experiments are performed with MATLAB 2012b on a machine with 8$\times$2.60 GHz CPUs and 32GB main memory. The average CPU time (in seconds) on each data set is shown in Figure~\ref{fig: running time}. We denote SVM implemented by the LIBLINEAR~\citep{Fan2008LIBLINEAR} package as SVM$_l$ and SVM implemented by SGD\footnote{http://leon.bottou.org/projects/sgd} as SVM$_s$, respectively. It can be seen that, both SVM$_s$ and our methods are faster than SVM$_l$, owing to the use of SGD. ODM$^L$ and ODM are just slightly slower than SVM$_s$ on three data sets (adult-a, w8a and mini-boo-ne) but highly competitive with SVM$_s$ on the rest data sets. Note that both SVM$_l$ and SVM$_s$ are very fast implementations of SVMs; this shows that ODM$^L$ and ODM are also computationally efficient.

\begin{figure}[!htb]
\begin{center}
\centerline{\includegraphics[width=\columnwidth]{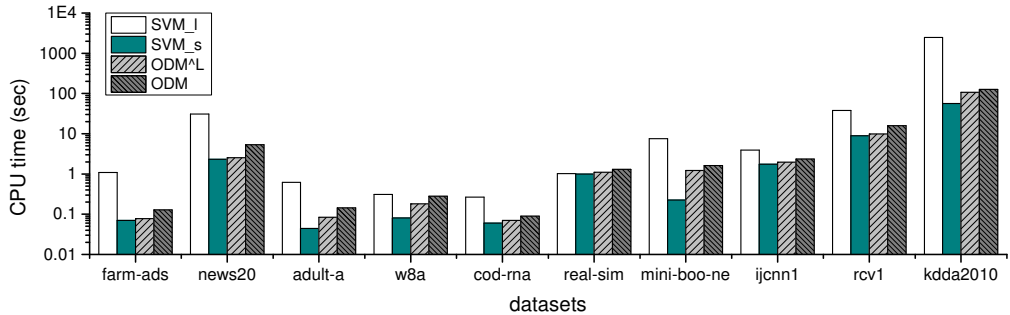}}
\caption{CPU time on the large scale data sets.}
\label{fig: running time}
\end{center}
\end{figure}

\section{Related Work} \label{sec: related work}

There are a few studies considered margin distribution in SVM-like algorithms~\citep{Garg2003Margin,Pelckmans2007A,Aiolli2008A}. Garg et al.~\citep{Garg2003Margin} proposed the Margin Distribution Optimization (MDO) algorithm which minimizes the sum of the cost of each instance, where the cost is a function which assigns larger values to instances with smaller margins. MDO can be viewed as a method of optimizing weighted margin combination, where the weights are related to the margins. The objective function optimized by MDO, however, is non-convex, and thus, it may get stuck in local minima. In addition, MDO can only be used for linear kernel.

Pelckmans et al.~\citep{Pelckmans2007A} proposed the Maximal Average Margin for Classifiers (MAMC) and it can be viewed as a special case of ODM$^L$ assuming that the margin variance is zero. MAMC has a closed-form solution, however, it will degenerate to a trivial solution when the classes are not with equal sizes.

Aiolli et al.~\citep{Aiolli2008A} proposed a Kernel Method for the direct Optimization of the Margin Distribution (KM-OMD) from a game theoretical perspective. Similar to MDO, this method also directly optimizes a weighted combination of margins over the training data, ignoring the influence of margin variances. Besides, this method considers hard-margin only, which may be another reason why it can't work well. It is noteworthy that the computational cost prohibits KM-OMD to be applied to large scale data.

The superiority of ODM$^L$ to the above methods have been presented in~\citep{Zhang2014Large}, so we don't choose them as compared methods in Section~\ref{sec: experiments}.

\section{Conclusions} \label{sec: conclusions}

Support vector machines work by maximizing the minimum margin. Recent theoretical results suggested that the margin distribution, rather than a single-point margin such as the minimum margin, is more crucial to the generalization performance. In this paper, we propose a new method, named Optimal margin Distribution Machine (ODM), which try to optimize the margin distribution by considering the margin mean and the margin variance simultaneously. Our method is a general learning approach which can be used in any place where SVM can be applied. Comprehensive experiments on thirty two regular scale data sets and ten large scale data sets validate the superiority of our method to SVM. In the future it will be interesting to generalize the idea of optimal margin distribution to other learning settings.

%

\bibliographystyle{elsarticle-harv}
\bibliography{Optimal_Margin_Distribution_Machine}

\begin{thebibliography}{37}
\expandafter\ifx\csname natexlab\endcsname\relax\def\natexlab#1{#1}\fi
\expandafter\ifx\csname url\endcsname\relax
  \def\url#1{\texttt{#1}}\fi
\expandafter\ifx\csname urlprefix\endcsname\relax\def\urlprefix{URL }\fi

\bibitem[{Aiolli et~al.(2008)Aiolli, San~Martino, and Sperduti}]{Aiolli2008A}
Aiolli, F., San~Martino, G., Sperduti, A., 2008. A kernel method for the
  optimization of the margin distribution. In: Proceedings of the 18th
  International Conference on Artificial Neural Networks. Prague, Czech, pp.
  305--314.

\bibitem[{Bordes et~al.(2009)Bordes, Bottou, and Gallinari}]{Bordes2009SGD-QN}
Bordes, A., Bottou, L., Gallinari, P., 2009. Sgd-qn: Careful quasi-newton
  stochastic gradient descent. Journal of Machine Learning Research 10,
  1737--1754.

\bibitem[{Bottou(2010)}]{Bottou2010Large-Scale}
Bottou, L., 2010. Large-scale machine learning with stochastic gradient
  descent. In: Proceedings of the 19th International Conference on
  Computational Statistics. Paris, France, pp. 177--186.

\bibitem[{Breiman(1999)}]{Breiman1999Prediction}
Breiman, L., 1999. Prediction games and arcing classifiers. Neural Computation
  11~(7), 1493--1517.

\bibitem[{Cortes and Vapnik(1995)}]{Cortes1995Support}
Cortes, C., Vapnik, V., 1995. Support-vector networks. Machine Learning 20~(3),
  273--297.

\bibitem[{Cotter et~al.(2013)Cotter, Shalev-shwartz, and
  Srebro}]{Learning2013Cotter}
Cotter, A., Shalev-shwartz, S., Srebro, N., 2013. Learning optimally sparse
  support vector machines. In: Proceedings of the 30th International Conference
  on Machine Learning. Atlanta, GA, pp. 266--274.

\bibitem[{Cristianini and Shawe-Taylor(2000)}]{Cristianini2000introduction}
Cristianini, N., Shawe-Taylor, J., 2000. An Introduction to Support Vector
  Machines and Other Kernel-based Learning Methods. Cambridge University Press,
  Cambridge, UK.

\bibitem[{Do and Alexandre(2013)}]{Convex2013Do}
Do, H., Alexandre, K., 2013. Convex formulations of radius-margin based support
  vector machines. In: Proceedings of the 30th International Conference on
  Machine Learning. Atlanta, GA, pp. 169--177.

\bibitem[{Fan et~al.(2008)Fan, Chang, Hsieh, Wang, and Lin}]{Fan2008LIBLINEAR}
Fan, R.~E., Chang, K.~W., Hsieh, C.~J., Wang, X.~R., Lin, C.~J., 2008.
  Liblinear: A library for large linear classification. Journal of Machine
  Learning Research 9, 1871--1874.

\bibitem[{Fisher(1936)}]{Fisher1936The}
Fisher, R.~A., 1936. The use of multiple measurements in taxonomic problems.
  Annals of Eugenics 7~(2), 179--188.

\bibitem[{Freund and Schapire(1995)}]{Freund1995A}
Freund, Y., Schapire, R.~E., 1995. A decision-theoretic generalization of
  on-line learning and an application to boosting. In: Proceedings of the 2nd
  European Conference on Computational Learning Theory. Barcelona, Spain, pp.
  23--37.

\bibitem[{Gao and Zhou(2013)}]{Gao2012On}
Gao, W., Zhou, Z.-H., 2013. On the doubt about margin explanation of boosting.
  Artificial Intelligence 199--200, 22--44.

\bibitem[{Garg and Roth(2003)}]{Garg2003Margin}
Garg, A., Roth, D., 2003. Margin distribution and learning algorithms. In:
  Proceedings of the 20th International Conference on Machine Learning.
  Washington, DC, pp. 210--217.

\bibitem[{Hsieh et~al.(2008)Hsieh, Chang, Lin, Keerthi, and
  Sundararajan}]{Hsieh2008A}
Hsieh, C.~J., Chang, K.~W., Lin, C.~J., Keerthi, S.~S., Sundararajan, S., 2008.
  A dual coordinate descent method for large-scale linear svm. In: Proceedings
  of the 25th International Conference on Machine Learning. Helsinki, Finland,
  pp. 408--415.

\bibitem[{Johnson and Zhang(2013)}]{Accelerating2013Johnson}
Johnson, R., Zhang, T., 2013. Accelerating stochastic gradient descent using
  predictive variance reduction. In: Burges, C. J.~C., Bottou, L., Welling, M.,
  Ghahramani, Z., Weinberger, K.~Q. (Eds.), Advances in Neural Information
  Processing Systems 26. Curran Associates, Inc., pp. 315--323.

\bibitem[{Jose et~al.(2013)Jose, Goyal, Aggrwal, and Varma}]{Jose2013Local}
Jose, C., Goyal, P., Aggrwal, P., Varma, M., 2013. Local deep kernel learning
  for efficient non-linear svm prediction. In: Proceedings of the 30th
  International Conference on Machine Learning. Atlanta, GA, pp. 486--494.

\bibitem[{Kushner and Yin(2003)}]{Kushner2003Stochastic}
Kushner, H.~J., Yin, G.~G., 2003. Stochastic approximation and recursive
  algorithms and applications; 2nd ed. Springer, New York.

\bibitem[{Lacoste-julien et~al.(2013)Lacoste-julien, Jaggi, Schmidt, and
  Pletscher}]{Lacoste2013Block}
Lacoste-julien, S., Jaggi, M., Schmidt, M., Pletscher, P., 2013.
  Block-coordinate frank-wolfe optimization for structural svms. In:
  Proceedings of the 30th International Conference on Machine Learning.
  Atlanta, GA, pp. 53--61.

\bibitem[{Luntz and Brailovsky(1969)}]{Luntz1969On}
Luntz, A., Brailovsky, V., 1969. On estimation of characters obtained in
  statistical procedure of recognition (in russian). Technicheskaya Kibernetica
  3.

\bibitem[{Pelckmans et~al.(2008)Pelckmans, Suykens, and Moor}]{Pelckmans2007A}
Pelckmans, K., Suykens, J., Moor, B.~D., 2008. A risk minimization principle
  for a class of parzen estimators. In: Platt, J., Koller, D., Singer, Y.,
  Roweis, S. (Eds.), Advances in Neural Information Processing Systems 20. MIT
  Press, Cambridge, MA, pp. 1137--1144.

\bibitem[{Polyak and Juditsky(1992)}]{Polyak1992Acceleration}
Polyak, B.~T., Juditsky, A.~B., 1992. Acceleration of stochastic approximation
  by averaging. SIAM Journal on Control and Optimization 30~(4), 838--855.

\bibitem[{Reddi et~al.(2015)Reddi, Hefny, Sra, Poczos, and Smola}]{Reddi2015On}
Reddi, S.~J., Hefny, A., Sra, S., Poczos, B., Smola, A., 2015. On variance
  reduction in stochastic gradient descent and its asynchronous variants. In:
  Cortes, C., Lawrence, N.~D., Lee, D.~D., Sugiyama, M., Garnett, R. (Eds.),
  Advances in Neural Information Processing Systems 28. Curran Associates,
  Inc., pp. 2647--2655.

\bibitem[{Reyzin and Schapire(2006)}]{Reyzin2006How}
Reyzin, L., Schapire, R.~E., 2006. How boosting the margin can also boost
  classifier complexity. In: Proceedings of 23rd International Conference on
  Machine Learning. Pittsburgh, PA, pp. 753--760.

\bibitem[{Schapire et~al.(1998)Schapire, Freund, Bartlett, and
  Lee}]{Schapire1998Boosting}
Schapire, R.~E., Freund, Y., Bartlett, P.~L., Lee, W.~S., 1998. Boosting the
  margin: a new explanation for the effectives of voting methods. Annuals of
  Statistics 26~(5), 1651--1686.

\bibitem[{Sch{\"o}lkopf and Smola(2001)}]{Scholkopf2001learning}
Sch{\"o}lkopf, B., Smola, A., 2001. Learning with kernels: support vector
  machines, regularization, optimization, and beyond. MIT Press, Cambridge, MA.

\bibitem[{Shalev-Shwartz et~al.(2007)Shalev-Shwartz, Singer, and
  Srebro}]{Shalev-Shwartz2007Pegasos}
Shalev-Shwartz, S., Singer, Y., Srebro, N., 2007. Pegasos: Primal estimated
  sub-gradient solver for svm. In: Proceedings of the 24th International
  Conference on Machine Learning. Helsinki, Finland, pp. 807--814.

\bibitem[{Shamir and Zhang(2013)}]{Shamir2013Stochastic}
Shamir, O., Zhang, T., 2013. Stochastic gradient descent for non-smooth
  optimization: Convergence results and optimal averaging schemes. In:
  Proceedings of the 30th International Conference on Machine Learning.
  Atlanta, GA, pp. 71--79.

\bibitem[{Takac et~al.(2013)Takac, Bijral, Richtarik, and
  Srebro}]{Takac2013Mini}
Takac, M., Bijral, A., Richtarik, P., Srebro, N., 2013. Mini-batch primal and
  dual methods for svms. In: Proceedings of the 30th International Conference
  on Machine Learning. Atlanta, GA, pp. 1022--1030.

\bibitem[{Vapnik(1995)}]{Vapnik1995The}
Vapnik, V., 1995. The Nature of Statistical Learning Theory. Springer-Verlag,
  New York.

\bibitem[{Wang et~al.(2011)Wang, Sugiyama, Yang, Zhou, and Feng}]{Wang2011A}
Wang, L.~W., Sugiyama, M., Yang, C., Zhou, Z.-H., Feng, J., 2011. A refined
  margin analysis for boosting algorithms via equilibrium margin. Journal of
  Machine Learning Research 12, 1835--1863.

\bibitem[{Xu(2011)}]{Xu2010Towards}
Xu, W., 2011. Towards optimal one pass large scale learning with averaged
  stochastic gradient descent. CoRR, abs/1107.2490.

\bibitem[{Yuan et~al.(2012)Yuan, Ho, and Lin}]{Yuan2012Recent}
Yuan, G.~X., Ho, C.~H., Lin, C.~J., 2012. Recent advances of large-scale linear
  classification. Proceedings of the IEEE 100~(9), 2584--2603.

\bibitem[{Zhang(2004)}]{Zhang2004Solving}
Zhang, T., 2004. Solving large scale linear prediction problems using
  stochastic gradient descent algorithms. In: Proceedings of the 21st
  International Conference on Machine learning. Banff, Canada, pp. 116--123.

\bibitem[{Zhang and Zhou(2014)}]{Zhang2014Large}
Zhang, T., Zhou, Z.-H., 2014. Large margin distribution machine. In:
  Proceedings of the 20th ACM SIGKDD Conference on Knowledge Discovery and Data
  Mining. New York, NY, pp. 313--322.

\bibitem[{Zhao and Zhang(2015)}]{Zhao2015Stochastic}
Zhao, P.-L., Zhang, T., 2015. Stochastic optimization with importance sampling
  for regularized loss minimization. In: Blei, D., Bach, F. (Eds.), Proceedings
  of the 32nd International Conference on Machine Learning. pp. 1--9.

\bibitem[{Zhou(2012)}]{Zhou2012Ensemble}
Zhou, Z.-H., 2012. Ensemble Methods: Foundations and Algorithms. CRC Press,
  Boca Raton, FL.

\bibitem[{Zhou(2014)}]{Zhou2014Large}
Zhou, Z.-H., 2014. Large margin distribution learning. In: Proceedings of the
  6th IAPR International Workshop on Artificial Neural Networks in Pattern
  Recognition (ANNPR'14). Montreal, Canada, pp. 1--11.

\end{thebibliography}

\end{document}